\documentclass[letterpaper, 10 pt, conference]{ieeeconf}  

\IEEEoverridecommandlockouts                              

\usepackage[T1]{fontenc}
\usepackage{lmodern}
\usepackage{anyfontsize}
\usepackage{newtxtext}
\usepackage{microtype}
\usepackage{soul}
\usepackage{siunitx}

\usepackage{graphicx}
\usepackage[dvipsnames]{xcolor}
\usepackage{nicematrix}
\usepackage{booktabs}
\usepackage{tabularx}
\usepackage{multirow}
\let\olditemize\itemize
\let\endolditemize\enditemize
\renewenvironment{itemize}{
  \olditemize
  \setlength{\leftmargin}{-2pt} 
  \setlength{\itemindent}{0pt}
}{
  \endolditemize
}

\let\oldenumerate\enumerate
\let\endoldenumerate\endenumerate
\renewenvironment{enumerate}{
  \oldenumerate
  \setlength{\leftmargin}{-1em} 
  \setlength{\itemindent}{0pt}
}{
  \endoldenumerate
}

\newcommand{\gauss}{{\scshape Gauss}}
\newcommand{\scshaper}{{\scshape r}}
\newcommand{\gausscvar}{{\scshape Gauss-Cva\scshaper}}
\newcommand{\vnb}{{\scshape Vnb}}
\newcommand{\cem}{{\scshape Cem}}
\newcommand{\pf}{{\scshape Pf}}

\usepackage{amsmath, amssymb, amsfonts,mathtools}

\usepackage{amsthm}
\usepackage{capt-of}
\usepackage[export]{adjustbox}

\newtheoremstyle{ieeeblock}
{6pt}      
{6pt}      
{\normalfont} 
{}         
{\bfseries} 
{.}        
{0.5em}    
{}        

\theoremstyle{ieeeblock}

\newtheorem{definition}{Definition}
\newtheorem{theorem}{Theorem}
\newtheorem{problem}{Problem}
\newtheorem{remark}{Remark}

\usepackage[ruled, vlined]{algorithm2e}
\usepackage[bookmarks=true,hyperfootnotes=false,colorlinks=true,allcolors=blue,urlcolor=blue]{hyperref}
\usepackage{bookmark}
\usepackage{cleveref}

\usepackage{titlesec}
\renewcommand{\thesection}{\Roman{section}}
\renewcommand{\thesubsection}{\Alph{subsection}}
\renewcommand{\thesubsubsection}{\arabic{subsubsection})}
\titleformat{\section}[block]
  {\normalfont\normalsize\scshape\centering}
  {\thesection.}{0.55em}{}
\titleformat{\subsection}[block]
  {\normalfont\itshape}
  {\thesubsection.}{0.5em}{}
\titleformat{\subsubsection}[runin]
  {\normalfont\itshape}
  {\thesubsubsection}{0.35em}{}[:]
\titlespacing*{\section}{0pt}{8pt plus 1pt minus 1pt}{3pt plus 1pt minus 1pt}
\titlespacing*{\subsection}{0pt}{5pt plus 1pt minus 1pt}{1pt plus 1pt minus 1pt}
\titlespacing*{\subsubsection}{0pt}{3pt plus 1pt minus 1pt}{0.5em}
\titlespacing*{\paragraph}{0pt}{2pt plus 1pt minus 1pt}{0.5em}
\makeatletter
\renewcommand{\p@subsection}{\thesection.}
\renewcommand{\p@subsubsection}{\thesection.\Alph{subsection}.}
\makeatother

\crefformat{appendix}{Appendix~#2#1#3}
\Crefformat{appendix}{Appendix~#2#1#3}

\usepackage{nicefrac}

\usepackage{array}
\newcolumntype{L}[1]{>{\raggedright\arraybackslash}p{#1}}
\newcolumntype{C}[1]{>{\centering\arraybackslash}p{#1}}
\newcolumntype{R}[1]{>{\raggedleft\arraybackslash}p{#1}}
\definecolor{OceanGreen}{HTML}{448F65} 
\definecolor{nodeBlue}{HTML}{D0E4F5}
\definecolor{nodeGreen}{HTML}{C8E6C9}
\definecolor{nodeAmber}{HTML}{FFE0B2}
\definecolor{nodeGray}{HTML}{ECEFF1}
\definecolor{txtDark}{HTML}{263238}
\definecolor{gradTeal}{HTML}{0C443E}
\definecolor{borderDark}{HTML}{000000}
\definecolor{simBg}{HTML}{F0F4F8}

\usepackage{tikz}
\usetikzlibrary{spy, shapes, arrows.meta, positioning, calc, fit}
\usepackage[mode=buildnew]{standalone}
\definecolor{nodeGrayHW}{HTML}{40413B}    
\definecolor{robotorange}{HTML}{E44C0D} 
\title{\textbf{Variational Neural Belief Parameterizations for \\ Robust Dexterous Grasping under Multimodal Uncertainty}}

\author{
  Clinton Enwerem\textsuperscript{1},\, Shreya Kalyanaraman\textsuperscript{2},\, John S. Baras\textsuperscript{1},\, Calin Belta\textsuperscript{1}\thanks{\textsuperscript{1}Department of Electrical \& Computer Engineering and Institute for Systems Research, University of Maryland, College Park, MD, USA. \textsuperscript{2}Maryland Applied Graduate Engineering, A.~James Clark School of Engineering, University of Maryland, College Park, MD, USA. Emails: \texttt{\{{enwerem, shreya05, baras, calin}\}@umd.edu}.}\thanks{\copyright~2026 IEEE. Personal use of this material is permitted.
Permission from IEEE must be obtained for all other uses, in any current
or future media, including reprinting/republishing this material for
advertising or promotional purposes, creating new collective works, for
resale or redistribution to servers or lists, or reuse of any copyrighted
component of this work in other works. Accepted for publication at the
2026 IEEE/RSJ International Conference on Intelligent Robots and Systems
(IROS).}
}

\begin{document}
\maketitle

\begin{abstract}
Contact variability, sensing uncertainty, and external disturbances make grasp execution stochastic. Expected-quality objectives ignore tail outcomes and often select grasps that fail under adverse contact realizations. Risk-sensitive POMDPs address this failure mode, but many use particle-filter beliefs that scale poorly, obstruct gradient-based optimization, and estimate Conditional Value-at-Risk (CVaR) with high-variance approximations. We instead formulate grasp acquisition as variational inference over latent contact parameters and object pose, representing the belief with a differentiable Gaussian mixture. We use Gumbel-Softmax component selection and location-scale reparameterization to express samples as smooth functions of the belief parameters, enabling pathwise gradients through a differentiable CVaR surrogate for direct optimization of tail robustness. In simulation, our variational neural belief improves robust grasp success under contact-parameter uncertainty and exogenous force perturbations while reducing planning time by roughly an order of magnitude relative to particle-filter model-predictive control. On a serial-chain robot arm with a multifingered hand, we validate grasp-and-lift success under object-pose uncertainty against a Gaussian baseline. Both methods succeed on the tested perturbations, but our controller terminates in fewer steps and less wall-clock time while achieving a higher tactile grasp-quality proxy. Our learned belief also calibrates risk more accurately, keeping mean absolute calibration error below 0.14 across tested simulation regimes, compared with 0.58 for a Cross-Entropy Method planner. We provide code, simulation assets, and a dataset of 243 force-closed grasps at the following link: \url{www.github.com/coenwerem/vnb-grasp}.
\end{abstract}

\section{Introduction}\label{sec:intro}

\begin{figure}[htb]
\centering
\vspace{-10pt}
\includegraphics[width=\linewidth]{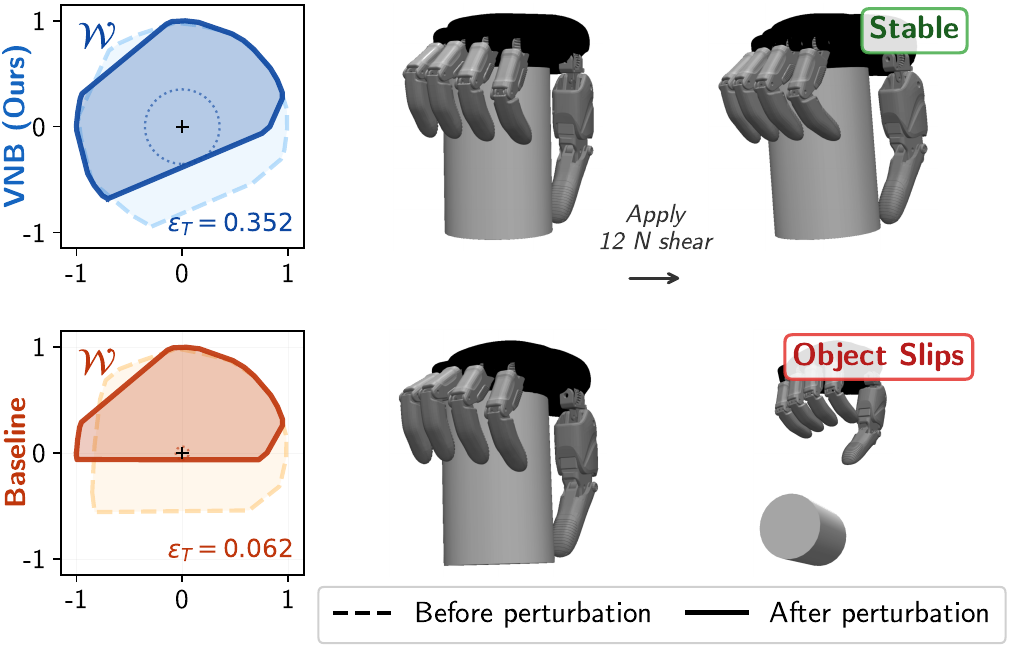}
\caption{This paper develops a variational neural belief (VNB) parameterization for robust grasping, casting grasping as a sequential risk-sensitive decision process over an online contact-geometry belief. The \emph{grasp wrench space} (GWS), \(\mathcal{W}\), (\emph{leftmost column}; axes omitted) is the convex hull of resistible hand wrenches. The dotted circle marks the terminal inscribed-ball radius, \(\varepsilon_T\), reported in the bottom-right corner of each GWS figure. Larger hulls containing the origin indicate greater disturbance rejection. VNB (top row) iteratively refines contact, producing a force-closed GWS that persists after a \SI{12}{N} lateral shear force is applied at the object's center of mass. A baseline planner (bottom row) commits to a single offline force-closed grasp before perturbation, but its GWS collapses under the same force, causing object slip. See \Cref{sec:resdis} for quantitative results.}
\label{fig:method_desc}
\end{figure}

Consider a robotic manipulator fitted with a multi-fingered hand and an externally fixed camera tasked with grasping an object at an unknown pose. The robot faces an epistemic problem: contact parameters that govern grasp stability, namely friction coefficients, surface compliance, and contact modes, are inexact and can only be inferred indirectly from noisy tactile and proprioceptive feedback. Additionally, the accuracy of the computed object pose may be corrupted by camera calibration drift, lighting, or partial camera observations. These sources of uncertainty do not merely constitute a nuisance that may simply be averaged out, since a grasp whose \emph{expected} quality metric is high under a certain object pose and friction prior may still fail catastrophically if the true pose and friction lie in the tails of their respective probability distributions \cite{weisz_pose_2012,liAnalyticTheoryIntrinsic2024a}. Robust grasping therefore demands planning algorithms that explicitly reason about the full posterior distribution over contact and pose uncertainty, and not only its mean.

 Mainstream analytical grasp synthesis pipelines largely ignore uncertainty. Optimization-based planners~\cite{miller_graspit_2004,li_tpgp_2024,liuSynthesizingDiversePhysically2022a} compute penetration-free, force-closure grasps under assumed contact properties, yet the resulting grasps are brittle to friction mismatch and pose error \cite{weisz_pose_2012}. Data-driven methods \cite{newburyDeepLearningApproaches2023,khargonkarNeuralGraspsLearningImplicit2022,wangDexGraspNetLargeScaleRobotic2023} learn grasp proposals from large datasets but optimize expected performance and provide no mechanism to quantify or hedge against contact or pose risk. Methods that explicitly model uncertainty, including Deep Variational Bayes Filters \cite{karl2017deep} and probabilistic grasp learning frameworks such as Dex-Net 2.0 \cite{mahler2017dexnet}, maintain a belief over uncertain parameters. A dominant nonparametric belief representation is the weighted empirical measure defined by a particle set, which enables approximate belief-space optimization~\cite{kaelbling_planning_1998,thrun2005probabilistic}. However, particle-based beliefs inherit several limitations from particle filtering that complicate tail risk estimation. First, standard risk measures such as the Conditional Value-at-Risk (CVaR) require accurate tail estimation, but finite particle sets introduce variance unless particle counts are increased at proportional computational cost \cite{chengImprovedParticleFilter2024}. Second, resampling introduces discrete stochasticity that prevents pathwise gradient propagation. Third, particle allocation limits posterior estimation fidelity, especially in multimodal settings \cite{vermaakMaintainingMultimodalityMixture2003}.

We address these limitations using a \emph{variational neural belief} represented as a continuous distribution over contact parameters with learned transition and observation updates. By sampling Gaussian mixture beliefs with Gumbel-Softmax component selection~\cite{jangCategoricalReparameterizationGumbelSoftmax2017} and location-scale reparameterization~\cite{kingma2013auto}, we obtain pathwise-differentiable samples with respect to the belief parameters. We then evaluate a smooth CVaR surrogate on these samples, propagating gradients through the belief while avoiding discontinuous resampling and the high-variance tail estimates induced by particle-based beliefs~\cite{hongMonteCarloMethods2014}.

Our contributions are as follows:
\vspace{-2pt}
\begin{enumerate}
\renewcommand{\labelenumi}{\roman{enumi}.}
\setlength{\itemsep}{0pt}
\setlength{\parskip}{0pt}
\item \emph{Differentiable Gaussian Mixture Belief}: A Gaussian Mixture Model (GMM) belief representation whose reparameterized samples admit pathwise gradients with respect to all distribution parameters (Section~\ref{sec:belief_nets}).
\item \emph{Risk-Sensitive Grasp Optimization}: A smooth CVaR-based grasp quality objective with a pathwise gradient estimator through the reparameterized belief (Section~\ref{sec:riskgrad}).
\item \emph{Neural Belief Dynamics}: Learned transition and observation models for differentiable belief updates (Section~\ref{sec:bel_dyn}).
\item \emph{Robust Grasping as Information-Guided Risk-Sensitive MPC}: Grasping as a belief-space model-predictive control (MPC) problem that selects finger motion primitives to incrementally increase grasp quality while minimizing failure risk, aiming to achieve force closure under the current belief. Our risk-aware belief optimization framework reshapes the grasp wrench space during execution (see \Cref{fig:method_desc}), expanding the set of disturbances the grasp can resist (\Cref{ssec:varmpc}).
\end{enumerate}
We evaluate dexterous grasping under friction and stiffness uncertainty in MuJoCo across multiple regimes and perturbations, and validate hardware performance under pose uncertainty (Section~\ref{sec:exps}).

\begin{figure}[t]
\centering
\setlength{\fboxsep}{0pt}
\setlength{\fboxrule}{1pt}
\begin{minipage}[c]{0.78\linewidth}
\centering
\begin{minipage}[c]{0.23\linewidth}
    \centering
    \fbox{\includegraphics[width=\linewidth]{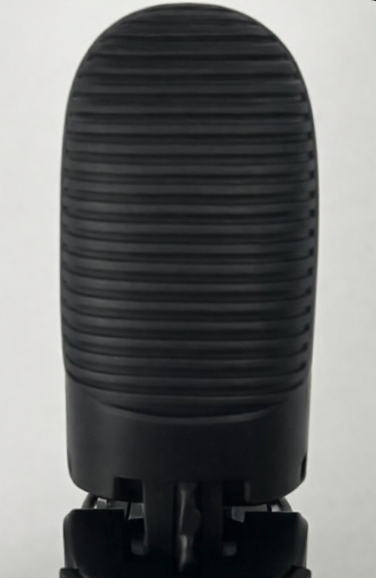}}
    \par\vspace{2pt}
    {\scriptsize\sffamily Piezoresistive\\Tactile Sensor}
\end{minipage}
\hfill
\begin{minipage}[c]{0.75\linewidth}
    \centering
    \begin{tikzpicture}[
        remember picture,
        lbl/.style={fill=white,draw=black,rounded corners=2pt,
                    inner sep=3pt,font=\scriptsize\sffamily},
        lblboxgray/.style={fill=nodeGrayHW,draw=nodeGrayHW,rounded corners=0pt,
                           inner sep=3.5pt,font=\scriptsize\sffamily},
        arr/.style={-{Latex[length=3mm,width=2mm,fill=none]},thin,robotorange!90!white}
    ]
    \node[anchor=south west,inner sep=0] (img) at (0,0)
        {\fbox{\includegraphics[width=\dimexpr\linewidth-2\fboxrule\relax]{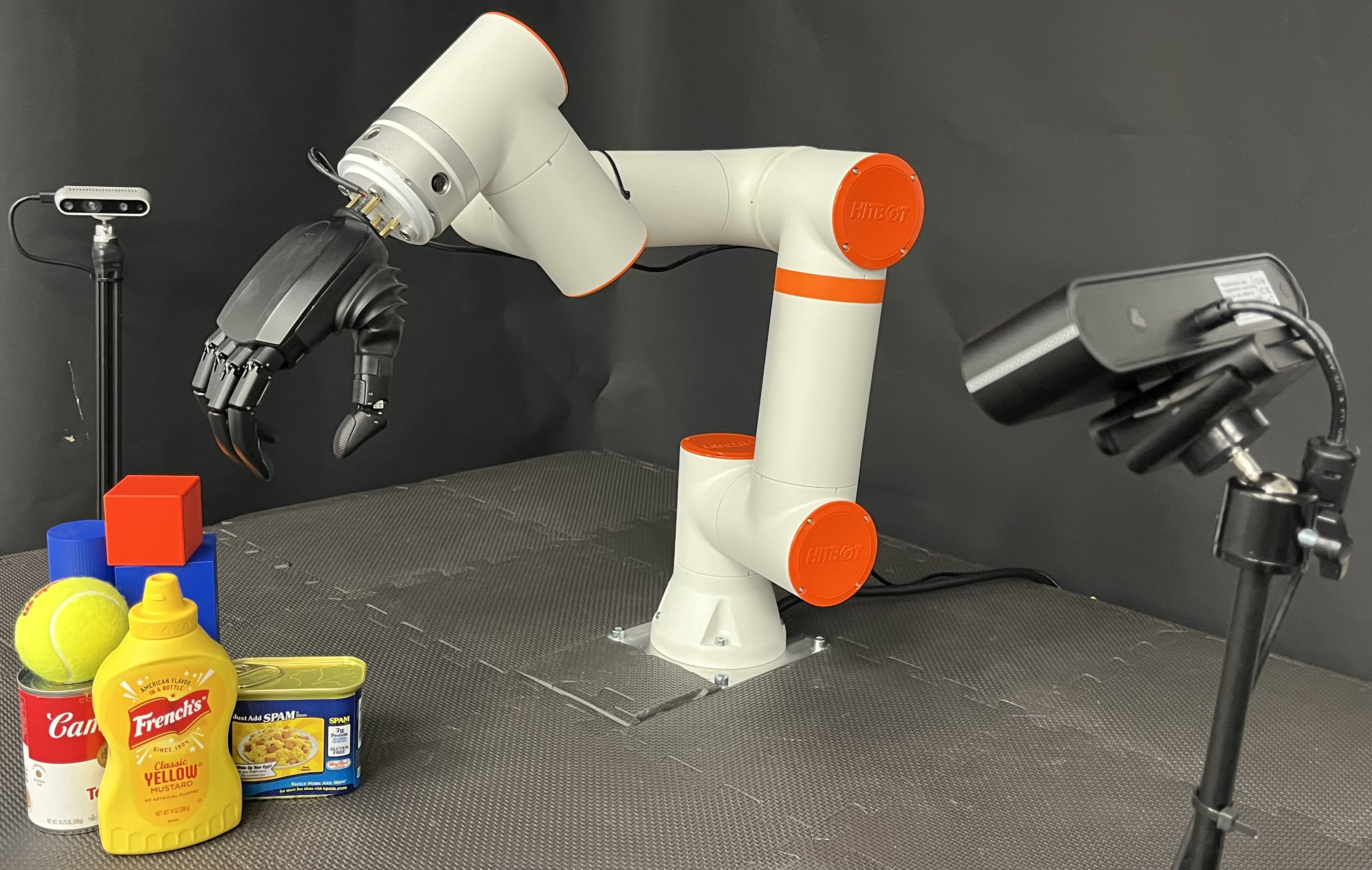}}};
    \begin{scope}[x={(img.south east)},y={(img.north west)},overlay]
        \coordinate (realsense-tgt) at (0.12,0.8);
        \node[lbl] (realsense-lbl) at (0.14,0.95) {RealSense D435i};
        \draw[arr] (realsense-lbl.south) -- (realsense-tgt);
        \node[lblboxgray] at (0.054,0.51) {};
        \node[lblboxgray] at (0.054,0.57) {};
        \coordinate (thumb) at (0.26,0.5);
        \path[draw=black,very thin,robotorange!90!white,even odd rule]
            (thumb) circle[radius=1.7mm];
        \draw[arr] ([shift=(157.4:0.8mm)] thumb) -- ++(-0.36,0.15);
        \coordinate (arm-tgt) at (0.6,0.56);
        \node[lbl,align=center] (arm-lbl) at (0.82,0.88) {Robot Arm\\(6 DoF)};
        \draw[arr] (arm-lbl.south west) -- (arm-tgt);
        \coordinate (orbbec-tgt) at (0.90,0.36);
        \node[lbl] (orbbec-lbl) at (0.75,0.14) {Orbbec Astra Pro Plus};
        \draw[arr] (orbbec-lbl.north) -- (orbbec-tgt);
        \coordinate (hand-tgt) at (0.29,0.69);
        \node[lbl,align=center] (hand-lbl) at (0.34,0.35) {Robot Hand\\(11 DoF)};
        \draw[arr] (hand-lbl.north) -- (hand-tgt);
        \coordinate (obj-tgt) at (0.22,0.18);
        \node[lbl] (obj-lbl) at (0.12,0.08) {Object Set};
        \draw[arr] (obj-lbl.north) -- (obj-tgt);
    \end{scope}
    \end{tikzpicture}
\end{minipage}
\end{minipage}
\caption{\textbf{Hardware Platform}: Our platform comprises a FAIR Innovation FR3 cobot (6 DoF), RealHand L6 robotic hand (11 DoF), two calibrated RGB-D cameras (RealSense D435i and Orbbec Astra Pro Plus), and representative primitives and YCB objects. We compare VNB-MPC with a Gaussian baseline under object-pose uncertainty and report results in \Cref{sec:resdis}, with hardware grasps shown in Fig.~\ref{fig:hwgrasps}.}
\label{fig:hwsetup}
\end{figure}
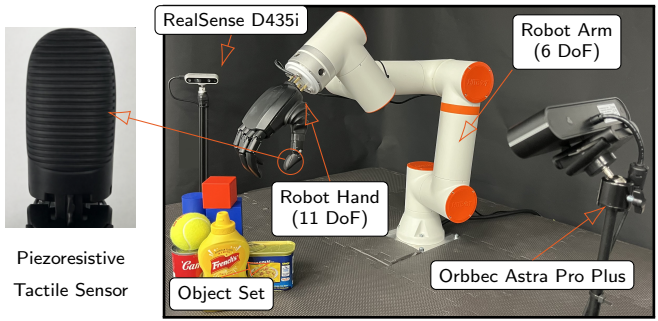

\subsection{Related Work}\label{sec:relwk}
\subsubsection{Belief-Space Planning for Dexterous Manipulation}
Belief-space planning represents uncertainty over latent variables such as object pose and contact parameters using probability distributions~\cite{platt2010belief,petrovskaya2011global}. Particle filters are widely used due to their flexibility and multimodal capacity~\cite{thrun2005probabilistic,dyro_particle_2021,vandenberglqgmp2011}, but provide discrete approximations that suffer from sample degeneracy and do not support pathwise gradient propagation. Risk-aware methods incorporate coherent measures such as CVaR to account for tail uncertainty~\cite{hakobyan2019risk}, yet particle-based beliefs require sampling-based gradient estimates. Our approach replaces particles with a continuous differentiable belief representation, enabling gradient-based optimization of risk-sensitive grasp objectives.

\subsubsection{Neural Uncertainty Representations}\label{ssec:neuunc}
Variational inference provides differentiable uncertainty representations optimized via gradient-based learning~\cite{blei2017variational}. Gaussian mixtures support multimodal modeling with reparameterized sampling~\cite{kingma2013auto}, and neural probabilistic models have been used to represent uncertainty in dynamics for planning~\cite{depeweg2017learning,chua2018deep}. Distributional reinforcement learning models return uncertainty~\cite{enwerem2025safety,dabney2018distributional}, but do not capture physical contact uncertainty. In contrast, our method represents uncertainty directly over contact and pose parameters, enabling multimodal belief updates and pathwise risk optimization beyond mean-performance Gaussian MPC~\cite{pan_probabilistic_2014}.

\section{Problem Formulation}\label{sec:prbform}
We consider dexterous grasping under partial observability, where a robotic manipulator 
equipped with a multi-fingered hand interacts with a stationary object (see \Cref{fig:hwsetup} for our hardware setup) whose pose and contact properties are uncertain. We assume an obstacle-free tabletop environment and a 6-DoF arm positioned at a fixed pre-grasp configuration; only the fingers move during grasp execution, and time is indexed discretely from \(t=0\). At each discrete time, \(t=0,1,\ldots\), the system applies an action \(\mathbf{a}_t \in \mathcal{A}\) (a 6-dimensional vector of finger joint velocity commands), receives an observation \(\mathbf{o}_t\) (comprising an RGB-D-derived pose estimate with covariance, per-finger contact forces from piezoresistive tactile sensors, and joint encoder readings), and maintains a belief over latent (non-observable) physical parameters governing grasp stability.

\emph{Notation}: We use \(\boldsymbol{\theta}\) for latent physical parameters, \(\boldsymbol{\phi}\) for Gaussian-mixture belief parameters, and \(\mathbf{w}\in\mathbb{R}^6\) for object-frame contact wrenches. The set of feasible object wrenches is denoted by \(\mathcal{W}\).

\begin{figure*}[t]
\centering
\includestandalone[width=.9\linewidth]{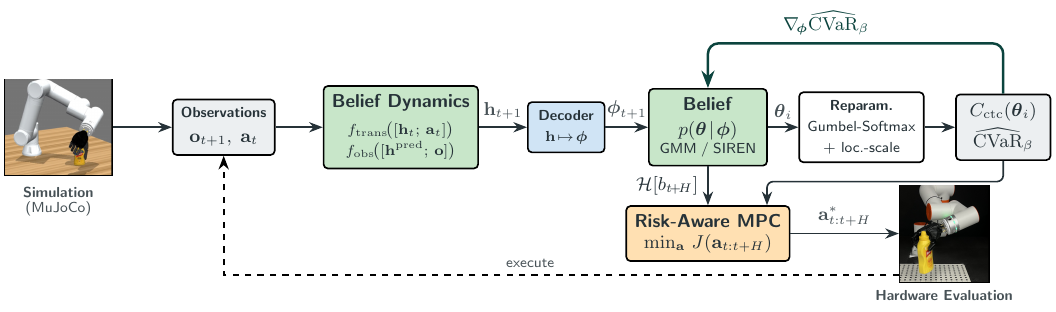}
\vspace{-10pt}
\caption{\textbf{Overview of Our Proposed Variational Neural Belief Grasping Framework.}
At each decision step, the neural belief dynamics update a latent embedding~\(\mathbf{h}_t\) via a prediction network~\(f_{\mathrm{trans}}\) conditioned on actions and a correction network~\(f_{\mathrm{obs}}\) conditioned on observations (joint angles, pose estimate with covariance, and contact geometry). We then evaluate grasp robustness under multimodal uncertainty along separate but connected simulation and hardware axes, with simulation providing controlled evaluation under known friction regimes, and hardware trials validating real-world transfer under object pose uncertainty (see \Cref{sec:exps}).}
\label{fig:frmwk}
\end{figure*}

\begin{definition}[Latent State]\label{def:latst}
The latent state at time \(t\) is
\begin{equation}\label{eq:thetadef}
\begin{array}{l}
\boldsymbol{\theta}_t =
\bigl(
\boldsymbol{\xi}_{o,t},
\boldsymbol{\psi}_t
\bigr)
\in \Theta,
\end{array}
\end{equation}
where \(\boldsymbol{\xi}_{o,t} \in SE(3)\) is the object pose, \(\boldsymbol{\psi}_t = \{\psi_t^i\}_{i=1}^{n_t^c}\), with \(\psi_t^i = (\mu_t^i,\kappa_t^i,d_t^i,s_t^i)\), comprises per-contact parameters at the \(n_t^c\) active contacts: friction coefficient \(\mu_t^i\), stiffness \(\kappa_t^i\), contact damping \(d_t^i\), and slip velocity magnitude \(s_t^i\), and \(\Theta\) is the set of possible values of latent states. The contact count satisfies \(0 \le n_t^c \le n_{\mathrm{max}}^c\), where \(n_{\mathrm{max}}^c \in \mathbb{N}\) is the maximum possible number of contacts. In our experiments, \(n_{\mathrm{max}}^c = 5\), matching the five piezoresistive tactile sensors on the robotic hand used in our hardware experiments (see \Cref{ssec:setup}).
\end{definition}

\noindent Since the latent state is not directly observable, the robot must instead reason over a probability distribution conditioned on its interaction history.

\begin{definition}[Belief State]\label{def:belief}
Let \(p(\cdot)\) denote the underlying probability density over latent parameters. The belief at time \(t\) is the posterior distribution
\begin{equation}
\begin{array}{l}
b_t(\boldsymbol{\theta}) = p(\boldsymbol{\theta}_t\mid\mathbf{o}_{1:t},\mathbf{a}_{0:t-1}),
\end{array}
\end{equation}
where \(\mathbf{o}_{1:t}\) and \(\mathbf{a}_{0:t-1}\) denote the observation and action histories.
\end{definition}

\noindent Observations comprise object pose estimates and contact-related signals. 
Specifically, contact-related signals include per-finger normal forces from piezoresistive tactile sensors, binary contact indicators, and joint position/velocity readings. Let \(\mathbf{o}_t = (\hat{\boldsymbol{\xi}}_t,\boldsymbol{\Sigma}_t^{\mathrm{pose}},\mathbf{y}^{\mathrm{tact}}_t,\mathbf{q}_t,\dot{\mathbf{q}}_t,\hat{o}_t,\hat{c}_t)\) denote the observation at time \(t\), where \(\hat{\boldsymbol{\xi}}_t \in SE(3)\) is the estimated object pose, \(\boldsymbol{\Sigma}_t^{\mathrm{pose}} \in \mathbb{R}^{6\times 6}\) is the pose estimation covariance derived from the ICP registration residual (see Section~\ref{ssec:cameras}), \(\mathbf{y}^{\mathrm{tact}}_t\) collects tactile measurements, \(\mathbf{q}_t,\dot{\mathbf{q}}_t\) are joint positions and velocities, and \(\hat{o}_t,\hat{c}_t \in [0,1]\) are occlusion and segmentation confidence scores. For compactness, we factor the observation likelihood into a pose term, a tactile/contact term, and a visual-reliability term, as in
\begin{equation}\label{eq:obs_model}
\begin{aligned}
p(\mathbf{o}_t \mid \boldsymbol{\theta}_t) \propto\;&
\mathcal{N}\!\bigl(
\hat{\boldsymbol{\xi}}_t \mid
\boldsymbol{\xi}_{o,t},
\boldsymbol{\Sigma}_t^{\mathrm{pose}}
\bigr) \\
&\times\,
p(\mathbf{y}^{\mathrm{tact}}_t,\mathbf{q}_t,\dot{\mathbf{q}}_t
\mid \boldsymbol{\psi}_t) \\
&\times\,
p(\hat{o}_t,\hat{c}_t
\mid \boldsymbol{\theta}_t),
\end{aligned}
\end{equation}
where the Gaussian term models pose estimation noise and the final factor captures observation reliability due to occlusion and segmentation uncertainty. The occlusion score~\(\hat{o}_t\) is computed as the fraction of object point-cloud points occluded by the hand mesh, and the segmentation score~\(\hat{c}_t\) is the Intersection-over-Union between the predicted and ground-truth object masks (see \Cref{ssec:cameras}).

\begin{definition}[Grasp Wrench Space Under Uncertainty]\label{def:wrsp}
Let \(\boldsymbol{\theta}=(\boldsymbol{\xi}_o,\boldsymbol{\psi})\) collect
the latent contact geometry and friction parameters. For hand configuration
\(\mathbf{q}_t\) and latent state \(\boldsymbol{\theta}_t\), the grasp wrench
space is
\begin{equation}\label{eq:wset}
\mathcal{W}(\mathbf{q}_t,\boldsymbol{\theta}_t)=
\bigl\{\mathbf{G}(\mathbf{q}_t,\boldsymbol{\xi}_{o,t})\,\mathbf{f}
\mid \mathbf{f}\in\mathcal{F}(\mathbf{q}_t,\boldsymbol{\psi}_t)\bigr\},
\end{equation}
where \(\mathbf{G}(\mathbf{q},\boldsymbol{\xi}_o)\) is the grasp
map~\cite{murrayMathematicalIntroductionRobotic2017a}, the linear mapping
from contact forces to object wrenches, and
\(\mathcal{F}(\mathbf{q},\boldsymbol{\psi})\) is the admissible contact force
set defined by Coulomb friction cones and unilateral contact
constraints~\cite{Prattichizzo2008}.
\end{definition}

\noindent Since \(\boldsymbol{\theta}_t\) is uncertain under the belief \(b_t\), the wrench space is a random set.
\begin{definition}[Risk-Sensitive Grasp Cost]\label{def:rsqual}
Let \(C(\mathbf{q}_t, \boldsymbol{\theta}_t)\) denote the differentiable grasp cost functional (detailed in Section~\ref{sec:riskgrad}), which combines contact stability, friction margins, and slip penalties. The risk-sensitive grasp cost is defined as
\begin{equation}
\begin{array}{l}
J(\mathbf{q}_t; b_t) = \rho_\beta\big(C(\mathbf{q}_t, \boldsymbol{\theta})\big),
\quad
\boldsymbol{\theta} \sim b_t,
\end{array}
\end{equation}
where \(\rho_\beta\) is a coherent risk measure. We use CVaR (Conditional Value-at-Risk), which averages the worst \((1{-}\beta)\) fraction of outcomes~\cite{rockafellar2000optimization} (see Section~\ref{sec:riskgrad} for a formal definition).
\end{definition}
\noindent Definition~\ref{def:rsqual} captures worst-case grasp cost over belief mass, rather than expected performance alone.

\begin{problem}[Risk-Sensitive Belief-Space Grasping]\label{prob:belctrl}
Determine an action sequence \(\mathbf{a}_{0:T-1}\)
that minimizes the cumulative risk-sensitive grasp cost subject to belief dynamics \(b_{t+1} =\mathcal{B}(b_t,\mathbf{a}_t,\mathbf{o}_{t+1})\), where \(\mathcal{B}\) is the belief update operator realized by Algorithm~\ref{alg:belupd} (Section~\ref{sec:bel_dyn}):
\vspace{-4pt}
\begin{equation}\label{eq:riskobj}
\begin{array}{l}
\min_{\mathbf{a}_{0:T-1}}
\sum_{t=0}^{T-1}
\rho_\beta\!\big(C(\boldsymbol{\theta}_t,\mathbf{a}_t)\big)
\end{array}
\end{equation}
\end{problem}

\noindent We assume that the initial belief \(b_0\) and the bounded problem horizon \(T \le T_{\mathrm{max}} \in \mathbb{N}\) are known. Problem~\ref{prob:belctrl} defines a risk-sensitive belief-space optimal control problem, where optimal actions must hedge against latent contact uncertainty. Solving this problem requires a belief representation that supports efficient sampling and optimization of risk-sensitive objectives. We address this challenge by introducing a differentiable variational belief parameterization next.

\section{Approach: Variational Neural Belief Networks}\label{sec:belief_nets}
Our solution (\Cref{fig:frmwk}) represents the belief as a Gaussian mixture whose samples are differentiable functions of the distribution parameters via Gumbel-Softmax and location-scale reparameterization. This reparameterization closes the gradient gap left open by particle filters, connecting the CVaR objective directly to the belief's sufficient statistics. We learn belief dynamics from data with neural networks, and optimize actions over a receding horizon with a gradient-based MPC planner.
\subsection{Gaussian Mixture Beliefs}\label{sec:gmm_belief}
We parameterize the belief as a mixture of \(K\) diagonal-covariance Gaussians
\begin{equation}\label{eq:gmm}
\begin{array}{l}
    p(\boldsymbol{\theta} \mid \boldsymbol{\phi}) = \sum_{k=1}^{K} \pi_k\, \mathcal{N}\!\bigl(\boldsymbol{\theta} \;\big|\; \boldsymbol{\mu}_k,\, \mathrm{diag}(\boldsymbol{\sigma}_k^2)\bigr),
\end{array}
\end{equation}
where \(\boldsymbol{\phi} = \{\boldsymbol{\ell}, \{\boldsymbol{\mu}_k, \log\boldsymbol{\sigma}_k\}_{k=1}^K\}\) collects the mixture logits \(\boldsymbol{\ell} = [\ell_1, \ell_2, \ldots, \ell_K]^\top \in \mathbb{R}^K\) and per-component means and log-standard-deviations. In typical mixture sampling, one first draws a component index, \(k \sim \mathrm{Categorical}(\boldsymbol{\pi})\), and then samples \(\boldsymbol{\theta} \sim \mathcal{N}(\boldsymbol{\mu}_k, \mathrm{diag}(\boldsymbol{\sigma}_k^2))\). However, since this discrete component selection is non-differentiable, we restore gradient flow through two mechanisms applied jointly, namely the Gumbel-Softmax relaxation~\cite{jangCategoricalReparameterizationGumbelSoftmax2017} of the categorical draw and the standard location-scale reparameterization. The former produces relaxed assignment weights, \(\boldsymbol{\zeta}\) that are differentiable with respect to the logits, with \(\boldsymbol{\zeta}=\left[\zeta_1, \ldots, \zeta_K\right]^{\top} \in \Delta^{K-1}\), where \(\Delta^{K-1}\) is the \((K{-}1)\)-dimensional probability simplex in \(\mathbb{R}^K\), i.e., the set \(\{\boldsymbol{\zeta}\in\mathbb{R}^K_{\ge0}\colon \sum_k \zeta_k = 1\}\). The latter proceeds by setting \(\boldsymbol{\theta}_k = \boldsymbol{\mu}_k + \boldsymbol{\sigma}_k \odot \boldsymbol{\epsilon}\), \(\boldsymbol{\epsilon} \sim \mathcal{N}(\mathbf{0}, \mathbf{I})\), making each component sample differentiable with respect to \(\boldsymbol{\mu}_k\) and \(\log \boldsymbol{\sigma}_k\) (\(\odot\) denotes the element-wise Hadamard product). 

Applying the relaxation and location-scale reformulation steps, we can then write a single reparameterized GMM sample as
\begin{equation}\label{eq:rsample}
\begin{array}{l}
\boldsymbol{\theta}
=\sum_{k=1}^{K}
\zeta_k \bigl(\boldsymbol{\mu}_k + \boldsymbol{\sigma}_k \odot \boldsymbol{\epsilon}_k\bigr),\quad\boldsymbol{\epsilon}_k \sim \mathcal{N}(\mathbf{0}, \mathbf{I}),
\end{array}
\end{equation}
where the weights \(\zeta_k\) are obtained via the Gumbel-Softmax relaxation
\begin{equation}\label{eq:gumbel_softmax}
\zeta_k =
\frac{\exp\!\left((\ell_k + g_k)/\tau\right)}
{\sum_{j=1}^{K} \exp\!\left((\ell_j + g_j)/\tau\right)},
\quad
g_k \sim \mathrm{Gumbel}(0,1),
\end{equation}
with temperature \(\tau>0\) controlling exploration--exploitation: low \(\tau\) yields near-deterministic component selection, while high \(\tau\) blends hypotheses. Every operation in~\eqref{eq:rsample} is differentiable with respect to \(\boldsymbol{\ell}\), \(\boldsymbol{\mu}_k\), and \(\log \boldsymbol{\sigma}_k\)\footnote{For highly multimodal posteriors that resist mixture approximation, our framework also supports an implicit neural belief using sinusoidal representation networks (SIREN)~\cite{sitzmann2020implicit} with Langevin sampling \cite{xuGlobalConvergenceLangevin2018}; see \Cref{apx:belief_nets}.}.

\subsection{Neural Belief Dynamics}\label{sec:bel_dyn}
Instead of specifying a hand-crafted process and observation model, we learn differentiable belief dynamics, with the belief defined by a latent embedding \(\mathbf{h}_t \in \mathbb{R}^{d_h}\) that is propagated through a neural transition model, and subsequently corrected using observations
\begin{equation}
\begin{array}{l}
\mathbf{h}_{t+1}^{\mathrm{pred}} = f_{\mathrm{trans}}\!\bigl([\mathbf{h}_t;\mathbf{a}_t]\bigr),\\
\mathbf{h}_{t+1} = f_{\mathrm{obs}}\!\bigl([\mathbf{h}_{t+1}^{\mathrm{pred}};\mathbf{o}_{t+1}]\bigr) \label{eq:obs}
\end{array}
\end{equation}
where \(f_{\mathrm{trans}}\) and \(f_{\mathrm{obs}}\) are multi-layer networks. A decoder (encapsulated by the \(\mathsf{DecodeBelief}\) subroutine; see Table~\ref{tab:hyperparams_full} for the architecture) maps \(\mathbf{h}_{t+1}\) to belief parameters \(\boldsymbol{\phi}_{t+1}\), yielding the updated belief \(b_{t+1}\).

\section{Pathwise Risk Gradients via Action Optimization}
\label{sec:riskgrad}
In this section, we derive the gradient of the CVaR objective \eqref{eq:riskobj} with respect to the planned action, enabling direct gradient-based grasp optimization. \Cref{apx:riskgrad} gives the corresponding reparameterized-belief proof.

\subsection{Differentiable CVaR via the Dual Representation}\label{ssec:diffcvar}
Let \(C\colon \mathbb{R}^d \to \mathbb{R}\) be a differentiable cost depending on latent contact parameters \(\boldsymbol{\theta}\) and action \(\mathbf{a}\). The CVaR at confidence level \(\beta \in (0,1)\) admits the dual form~\cite{rockafellar2000optimization}:
\vspace{-4pt}
\begin{equation}\label{eq:cvar_dual}
\mathrm{CVaR}_\beta[C] = \min_{\eta\in\mathbb{R}} \left\{\eta + \frac{1}{1-\beta}
\mathbb{E}\!\left[\max\!\big(C(\boldsymbol{\theta},\mathbf{a})-\eta,0\big)\right]\right\},
\end{equation}
where the expectation in \eqref{eq:cvar_dual} is taken over realizations of its argument corresponding to random values of \(\boldsymbol{\theta}\) drawn from the belief distribution, \({b}\). We set \(\eta\) to the empirical \(\beta\)-quantile and approximate the hinge with a softplus over \(N\) stochastic contact parameter realizations under the current belief, with \(\hat{\eta}\) capturing the empirical VaR threshold that defines the tail region:
\begin{equation}\label{eq:cvar_soft}
\begin{array}{l}
\widehat{\mathrm{CVaR}}_\beta = \hat\eta + \frac{1}{(1-\beta)N} \sum_{i=1}^{N}
\mathsf{SoftPlus}(C(\boldsymbol{\theta}_i, \mathbf{a}) - \hat\eta;\kappa_\rho),
\end{array}
\end{equation}
where \(\boldsymbol{\theta}_i = g(\boldsymbol{\epsilon}_i,\boldsymbol{\phi})\) are reparameterized samples. The \(\mathsf{SoftPlus}(x ; \kappa_\rho)\) function, defined as \(\kappa_\rho^{-1} \log \left(1+e^{\kappa_{\rho} x}\right)\), is a smooth approximation of the hinge \((x)_{+}\), where \(\kappa_\rho\) controls the sharpness of the approximation. This makes \(\widehat{\mathrm{CVaR}}_\beta\) differentiable with respect to both \(\mathbf{a}\) and \(\boldsymbol{\phi}\).

\begin{remark}[\(\nabla_{\mathbf{a}}\,\widehat{\mathrm{CVaR}}_\beta\) Computation]\label{prop:action_grad}
For reparameterizable belief \(b(\boldsymbol{\phi})\) and differentiable cost \(C\), automatic differentiation yields:
\begin{equation}\label{eq:cvargrad}
\begin{array}{l}
\nabla_{\mathbf{a}} \widehat{\mathrm{CVaR}}_\beta = \frac{1}{(1-\beta)N} \sum_{i=1}^{N} \sigma_{\kappa}(C_i - \hat\eta)
\nabla_{\mathbf{a}} C(\boldsymbol{\theta}_i, \mathbf{a}),
\end{array}
\end{equation}
where \(C_i = C(\boldsymbol{\theta}_i,\mathbf{a})\) and \(\sigma_{\kappa}\) is the sigmoid with sharpness~\(\kappa_\rho\). The softplus bias \(\kappa_\rho^{-1}\log 2\) is negligible for \(\kappa_\rho=5\). Unlike particle filters, this formulation enables pathwise gradient-based action optimization without resampling.
\end{remark}

\subsection{Differentiable Multimodal Cost Function}
We define a differentiable cost over uncertain contact realizations. For
\(\boldsymbol{\theta}=(\boldsymbol{\xi}_o,\boldsymbol{\psi})\) and action
\(\mathbf{a}\), let \(\mathbf{q}(\mathbf{a})\) denote the hand configuration
reached by executing \(\mathbf{a}\) from the current state, and let
\(\mathcal{F}(\mathbf{q}(\mathbf{a}),\boldsymbol{\psi})\) and
\(\mathcal{W}(\mathbf{q}(\mathbf{a}),\boldsymbol{\theta})\) be as in
Definition~\ref{def:wrsp}. We evaluate grasp stability through the decomposition
\begin{equation}\label{eq:cost_all}
C(\boldsymbol{\theta},\mathbf{a}) = C_{\mathrm{str}}(\mathbf{a}) + C_{\rho}(\mathbf{a},\boldsymbol{\psi}) + C_{\mathrm{ctc}}(\boldsymbol{\psi}), 
\end{equation}
into closure strength (\(C_{\mathrm{str}}\)), friction-limited robustness (\(C_{\rho}\)), and contact stability (\(C_{\mathrm{ctc}}\)), respectively given by
\begin{equation}\label{eq:cost_dec}
\begin{aligned}
C_{\mathrm{str}}(\mathbf{a})
&= -\alpha_s\bar{a} - \alpha_g G_{\mathrm{close}}(\mathbf{a}), \\[-1pt]
C_{\rho}(\mathbf{a},\boldsymbol{\psi})
&= -\alpha_r\,\mathsf{ReLU}\bigl(\alpha_n\bar{a}
   - \min_i\mu_i^{-1}\bigr), \\[-1pt]
C_{\mathrm{ctc}}(\boldsymbol{\psi})
&= \frac{1}{n_t^c}\sum_{i=1}^{n_t^c}
   \!\Bigl[e^{-\mu_i}+e^{-\kappa_i}
   +\mathsf{SoftPlus}(s_i)\Bigr],
\end{aligned}
\end{equation}
 where \(\bar{a}=\frac{1}{n_j}\sum_j a_j\) is the mean joint closing rate with \(n_j=6\) actuated finger joints,
\(G_{\mathrm{close}}(\mathbf{a}) = -\sum_{i=1}^{n_j}\bigl\|\,p_i(\mathbf{q}(\mathbf{a})) - c_{\mathrm{obj}}\bigr\|\) is a differentiable closure score computed as the negative sum of fingertip-to-object-center distances from forward kinematics, and \((\mu_i,\kappa_i,s_i)\) denote friction, stiffness, and slip at contact \(i\). The positive weights (\(\alpha_s, \alpha_g, \alpha_r, \text{ and }\alpha_n\)) scale the respective cost terms (see Table~\ref{tab:hyperparams_full}).

\subsection{Visual Perception Cost}\label{sec:vis_cost}%
\noindent We define a perception cost \(C_{\mathrm{vis}}(\mathbf{o}_t)\) that penalizes visually-uncertain states
\vspace{-4pt}
\begin{equation}\label{eq:cost_vis}
\omega_{\mathrm{pose}}\,\mathrm{tr}(\boldsymbol{\Sigma}_t^{\mathrm{pose}}) + \omega_{\mathrm{occ}}\,\mathsf{SoftPlus}(\hat{o}_t) + \omega_{\mathrm{seg}}\,(1 - \hat{c}_t),
\end{equation}
where \(\boldsymbol{\Sigma}_t^{\mathrm{pose}}\), \(\hat{o}_t\), and \(\hat{c}_t\) are as in~\eqref{eq:obs_model}. We set \(\lambda_v = 0\) in simulation and evaluate \(C_{\mathrm{vis}}\) only in hardware trials.

\subsection{Failure Probability}\label{sec:failure_prob}
We define the belief-space failure probability as the mass assigned to latent parameters under which the grasp cost exceeds a stability threshold \(\tau_f\):
\vspace{-4pt}
\begin{equation}\label{eq:fail_prob}
\begin{array}{l}
\hat{P}_{\mathrm{fail}}(b_t)=\Pr_{\boldsymbol{\theta}\sim b_t}\!\big(C(\boldsymbol{\theta},\mathbf{a}_t)>\tau_f\big).
\end{array}
\end{equation}
In practice, we estimate \eqref{eq:fail_prob} via Monte Carlo using \(N\) samples
\begin{equation}\label{eq:fail_mc}
\begin{array}{l}
\hat{P}_{\mathrm{fail}} =
\frac{1}{N}
\sum_{i=1}^{N}
\mathbf{1}\!\left[C(\boldsymbol{\theta}_i, \mathbf{a}_t) > \tau_f\right],
\end{array}
\end{equation}
where \(\boldsymbol{\theta}_i \sim b_t\). For the variational belief, we replace the hard indicator with a sigmoid of sharpness~\(\kappa_f\) to preserve differentiability, following standard smooth surrogate constructions used in variational inference and risk-sensitive optimization~\cite{kingma2013auto,rockafellar2000optimization}:
\begin{equation}\label{eq:fail_soft}
\begin{array}{l}
\hat{P}_{\mathrm{fail}}^{\mathrm{soft}} =
\frac{1}{N}
\sum_{i=1}^{N}
\sigma\!\bigl(\kappa_f(C(\boldsymbol{\theta}_i,\mathbf{a}_t)-\tau_f)\bigr).
\end{array}
\end{equation}
In our framework, failure probability, CVaR, and belief entropy measure complementary aspects of robustness: \(\hat{P}_{\mathrm{fail}}\) quantifies the mass of unstable configurations, \(\mathrm{CVaR}_\beta\) their severity, and \(\mathcal{H}[b_t]\) belief spread or entropy.

\begin{algorithm}[t]
\caption{VNB--MPC for Grasping}\label{alg:vnbmpc}
\SetAlgoVlined
\DontPrintSemicolon

\KwIn{Belief params.\ \(\boldsymbol{\phi}_0\), CVaR level \(\beta\), failure bound \(\delta\), target \(\varepsilon\) (\(\varepsilon_{\mathrm{des}}\)), horizon \(H\), grad.\ steps \(M\), step size \(\alpha_u\), samples \(N\), max steps \(T_{\mathrm{max}}\)}

\(t\gets0,\ \varepsilon\gets0\)

\While{\(\varepsilon<\varepsilon_{\mathrm{des}}\) \textbf{and} \(t<T_{\mathrm{max}}\)}{

\textbf{Per-component action optimization}\;

\For{\(k=1,\dots,K\)}{

\(\{\boldsymbol{\theta}_i^{(k)}\}_{i=1}^N \gets g(\boldsymbol{\epsilon}_i,\boldsymbol{\phi}_k)\)\tcp*{reparam.~\eqref{eq:rsample}}

\(\mathbf{a}_{t:t+H-1}^{(k)}\gets\mathbf{a}_{\mathrm{init}}\)

\For{\(j=1,\dots,M\)}{

\(\mathbf{a}_{t:t+H-1}^{(k)}
\gets
\mathbf{a}_{t:t+H-1}^{(k)}
-\alpha_u\nabla_{\mathbf a}
\widehat{\mathrm{CVaR}}_\beta[C(\boldsymbol{\theta}^{(k)},\mathbf{a}_{t:t+H-1}^{(k)})]\)\tcp*{\eqref{eq:cvargrad}}

}

\(J^{(k)}\gets
\widehat{\mathrm{CVaR}}_\beta[C(\boldsymbol{\theta}^{(k)},\mathbf{a}_{t:t+H-1}^{(k)})]
+\gamma\,\mathcal{H}[b_{t+H}]\)\tcp*{\eqref{eq:mpcobj}}

\(\hat P_{\mathrm{fail}}^{(k)}\gets\) failure prob.\ via~\eqref{eq:fail_soft}

}

\(k^\star\gets\arg\min_k J^{(k)}\) \textbf{s.t.} \(\hat P_{\mathrm{fail}}^{(k)}\le\delta\)

Execute \(\mathbf a_t\gets\mathbf a_t^{(k^\star)}\), observe \(\mathbf o_{t+1}\)

\(\boldsymbol{\phi}_{t+1}\leftarrow
\mathsf{NeuralBeliefUpdate}(\boldsymbol{\phi}_t,\mathbf a_t,\mathbf o_{t+1})\)

\uIf{simulation}{
    \({\displaystyle\varepsilon\gets\min_{\|\mathbf{d}\|=1}\max_{\mathbf{w}\in\mathcal{W}}
    \mathbf{d}^\top\mathbf{w}}\)\tcp*{see~\cite{ferrari1992planning}}
}
\Else{
    \(\varepsilon\gets\hat{\varepsilon}:=\mathcal{Q}(\mathbf{o}_{t+1})\)\tcp*{hw.\ proxy~(App.~\ref{apx:ehatcomp})}
}

\(t\gets t+1\)

}

\Return grasp with quality \(\varepsilon\)

\end{algorithm}

\begin{algorithm}[t]
\caption{\(\mathsf{NeuralBeliefUpdate}\)}\label{alg:belupd}
\SetAlgoVlined
\DontPrintSemicolon
\KwIn{Latent state \(\mathbf{h}_t\), action \(\mathbf{a}_t\), observation \(\mathbf{o}_{t+1}\), EMA rate \(\alpha_{\mathrm{ema}}\)}
\KwOut{Updated belief \(b_{t+1}\) parameterized by \(\boldsymbol{\phi}_{t+1}\)}
\(\mathbf{h}_{t+1}^{\mathrm{pred}}
\leftarrow
f_{\mathrm{trans}}\!\left(
\begin{bmatrix}
\mathbf{h}_t,\,\mathbf{a}_t
\end{bmatrix}
\right)\)\tcp*[r]{Prediction~\eqref{eq:obs}}
\(\mathbf{h}_{t+1}
\leftarrow
f_{\mathrm{obs}}\!\left(
\begin{bmatrix}
\mathbf{h}_{t+1}^{\mathrm{pred}},\,\mathbf{o}_{t+1}
\end{bmatrix}
\right)\)\tcp*[r]{Correction~\eqref{eq:obs}}
\(\boldsymbol{\phi}_{t+1}
\leftarrow
(1-\alpha_{\mathrm{ema}})\,\boldsymbol{\phi}_t + \alpha_{\mathrm{ema}}\,\mathsf{DecodeBelief}(\mathbf{h}_{t+1})\)%

\Return \(b_{t+1}=p(\boldsymbol{\theta}_{t+1} \mid \boldsymbol{\phi}_{t+1})\)
\end{algorithm}

\subsection{MPC for Grasping with Variational Beliefs}\label{ssec:varmpc}
At the core of our method is a belief-space model-predictive control (MPC) grasping formulation. At each step, the controller selects an action sequence that minimizes a risk-aware objective combining contact risk, perception cost, and belief entropy
\begin{equation}\label{eq:mpcobj}
\begin{aligned}
J(\mathbf{a}_{t:t+H})
& = \lambda_c\,\mathrm{CVaR}_\beta[C_t] + (1-\lambda_c)\,\mathbb{E}_{\boldsymbol{\theta}\sim b_t}[C_t] \\
&\quad + \lambda_v\, C_{\mathrm{vis}} + \gamma\,\mathcal{H}[b_{t+H}],
\end{aligned}
\end{equation}
where \(C_t=C(\boldsymbol{\theta},\mathbf{a}_t)\), \(\lambda_c \in [0,1]\) blends expected cost (\(\lambda_c=0\)) with full CVaR (\(\lambda_c=1\)), \(\mathcal{H}[b_{t+H}]\) denotes the terminal belief entropy, encouraging information-gathering actions, \(C_{\mathrm{vis}}\) is the visual perception cost~\eqref{eq:cost_vis}, and the weights \(\lambda_v \ge 0\) and \(\gamma\ge0\) balance perception and exploration (we set \(\lambda_v=0\) in simulation experiments). To compute gradients of the contact-risk term with respect to the action, we differentiate through the sample mean and the soft CVaR approximation~\eqref{eq:cvar_soft} of the contact cost
\(C(\boldsymbol{\theta},\mathbf{a}_t)\)
\begin{equation}\label{eq:mpc_grad}
\begin{aligned}
\nabla_{\mathbf{a}_t} J_{\mathrm{contact}}
&=
\frac{1-\lambda_c}{N}\sum_{i=1}^{N}
\nabla_{\mathbf{a}_t} C_i \\
&\quad+
\frac{\lambda_c}{(1-\beta)N}\sum_{i=1}^{N}
\sigma_{\kappa}(C_i-\hat{\eta})
\nabla_{\mathbf{a}_t} C_i,
\end{aligned}
\end{equation}
where \(C_i=C(\boldsymbol{\theta}_i,\mathbf{a}_t)\), and \(\sigma_{\kappa}\) denotes a sigmoid with sharpness~\(\kappa\). The remaining differentiable terms in~\eqref{eq:mpcobj} add their standard gradients. This formulation enables direct gradient-based action optimization, providing finer control updates than sampling-based methods such as the Cross-Entropy Method (CEM)~\cite{pinneri_sample-efficient_2021}. To mitigate local minima we use multi-start optimization with three initial action magnitudes and select the action with the lowest risk-aware objective.

Given the MPC objective in \eqref{eq:mpcobj}, \Cref{alg:vnbmpc} presents our VNB--MPC solution routine that proceeds thus: the controller maintains a \(K\)-component Gaussian mixture belief over object pose and contact parameters. Because multimodal beliefs represent distinct contact hypotheses, we perform \emph{per-component} action optimization: the controller optimizes actions independently for each mixture component and selects the one with the lowest risk-aware cost. This avoids optimizing against an averaged belief and lets the controller hedge against pose ambiguity, where a single finger motion can both disambiguate hypotheses and avoid catastrophic contact errors. We write \(\hat{\varepsilon}=\mathcal{Q}(\mathbf{o})\) for the Ferrari--Canny grasp quality~\cite{ferrari1992planning} estimated from tactile observations (see \Cref{apx:ehatcomp}). Since the focus of this work is on belief representations, we have omitted detailed derivations of the underlying motion controller and assume a belief-space controller that produces kinematically feasible, self-collision-free finger motions given sampled pose and contact realizations from the belief.

\section{Experiments}\label{sec:exps}
\subsection{Simulation Sandbox and Hardware Setup}\label{ssec:setup}

\noindent\textbf{Simulation Environment:}
\noindent We evaluate \Cref{alg:vnbmpc} on a dexterous grasping task using an 11-DoF hand attached to a 6-DoF arm in MuJoCo. We sample contact stiffness from \(\kappa \sim \mathcal{N}(1000, 250^2)\)\,N/m, to yield stable simulation across all tested objects, and perturb object mass by \(\pm 20\%\) of the nominal value. MuJoCo computes the effective contact friction as the geometric mean of the friction coefficients of the finger and object colliding surfaces:
\begin{equation}\label{eq:mu_eff}
\begin{array}{l}
\mu_{\mathrm{eff}} = \sqrt{\mu_{\mathrm{f}} \cdot \mu_{\mathrm{o}}},
\end{array}
\end{equation}
where \(\mu_{\mathrm{f}}\) is the friction coefficient of the robotic hand's finger material (assumed to be fixed), and \(\mu_{\mathrm{o}}\) is a per-object friction coefficient sampled per friction regime. 
\smallskip

\noindent\textbf{Hardware Platform:}
Our hardware platform (\Cref{fig:hwsetup}) consists of a 6-DoF FAIR Innovation FR3 robot arm equipped with a RealHand Inc.\ L6 multi-fingered hand with piezoresistive tactile sensing. The hand provides joint position control and tactile feedback for contact detection and post-grasp evaluation. Our perception stack utilizes two RGB-D cameras with complementary viewpoints: an Orbbec Astra Pro Plus mounted laterally about 46\,cm from the workspace and a RealSense D435i mounted on the opposite side. The Astra provides the primary object pose estimate, while the D435i supplies a complementary estimate from a less occluded viewpoint before grasp execution. We place objects on a planar table within the arm's workspace.

\subsection{Baselines}\label{ssec:baselines}
\noindent We compare our approach against four baselines spanning risk-neutral and risk-aware belief-space MPC. \emph{Gaussian MPC} ({\gauss}) uses a single Gaussian belief (\(K{=}1\)) with an expected-cost objective (\(\lambda_c = 0\) in~\eqref{eq:mpcobj}). \emph{Gaussian-CVaR MPC} ({\gausscvar}) extends \gauss\ with a CVaR objective~\eqref{eq:cvar_soft} (\(\lambda_c = 0.5\)) while retaining a single Gaussian belief and location-scale reparameterization for pathwise gradients. \emph{Particle Filter} ({\pf}) MPC uses a particle-filter belief (\(N{=}100\)) with sampling-based optimization and CVaR estimated via particle sorting and truncation. \emph{CEM MPC} ({\cem}) follows~\cite{pinneri_sample-efficient_2021}, using a Gaussian belief, population size 64, elite fraction 0.2, and 3 CEM iterations. Our method, \emph{Variational Neural Belief MPC} ({\vnb}), uses a reparameterizable GMM belief (\(K{=}8\)) and the full risk-aware objective~\eqref{eq:mpcobj} with \(\lambda_c = 1\).

\subsection{Multimodal Perception and Object Pose Estimation}\label{ssec:cameras}
Using the sensors in the hardware platform described in \Cref{ssec:setup}, we estimate the pose of the object to be grasped using ICP registration from the Orbbec Astra Pro Plus. The resulting ICP registration fitness and residual yield the pose covariance, \(\boldsymbol{\Sigma}_t^{\mathrm{pose}}\), which enters the MPC objective~\eqref{eq:mpcobj}. When both cameras observe the object, the two ICP hypotheses are fused using covariance intersection,
\begin{equation}\label{eq:cov_intersection}
\left(\boldsymbol{\Sigma}_t^{\mathrm{pose}}\right)^{-1}
=
\omega_{\mathrm{CI}}\boldsymbol{\Sigma}_1^{-1}
+
(1-\omega_{\mathrm{CI}})\boldsymbol{\Sigma}_2^{-1},
\quad \omega_{\mathrm{CI}}\in[0,1],
\end{equation}
with a fixed \(\omega_{\mathrm{CI}}=0.5\) in our implementation. During grasp execution and for the lift-and-hold evaluation phase, tactile observations replace vision for contact detection and tactile grasp-quality proxy evaluation (see \Cref{ssec:hwexp,apx:ehatcomp}).

\subsection{Grasp Stress Tests and Evaluation Metrics}\label{ssec:strmet}%
\noindent\textbf{Friction Regimes:}
To expose robustness differences between our approach and the baselines in \Cref{ssec:baselines}, we evaluate across three friction regimes of increasing difficulty. In the \emph{Nominal} regime, we sample object friction as \(\mu_{\mathrm{o}} \sim \mathcal{U}[0.4, 1.0]\), yielding \(\mu_{\mathrm{eff}} \in [0.37, 0.59]\). The \emph{Wide} regime spans the full uncertainty range with \(\mu_{\mathrm{o}} \sim \mathcal{U}[0.15, 1.2]\), yielding \(\mu_{\mathrm{eff}} \in [0.23, 0.65]\), exposing planners to both low- and high-friction contacts. The \emph{Bimodal} regime models multimodal contact conditions using a Gaussian mixture, \(\mu_{\mathrm{o}} \sim 0.5\,\mathcal{N}(0.18, 0.03^2) + 0.5\,\mathcal{N}(1.0, 0.05^2)\), producing a bimodal distribution over effective friction.

We select friction coefficients from Schneider \& Company's reference chart \cite{CoefficientFrictionReference} under dry-contact assumptions. The fingertip collision geometries\footnote{We initially set \(\mu_{\mathrm{f}}=0.35\) to match the PEEK fingertips of the real hand. In MuJoCo this setting produced unstable contacts for several objects, so we used the simulator default \(\mu_{\mathrm{f}}=1.0\) for controlled comparisons and varied object friction instead.} use a fixed \(\mu_{\mathrm{f}} = 1.0\). By~\eqref{eq:mu_eff}, the effective friction therefore depends on only the object's surface. \Cref{tab:muexps} summarizes the assumed object friction ranges and nominal values. We report results on two object primitives (a cube and a box) under four risk levels (\(\beta \in \{0.5, 0.9, 0.95, 0.99\}\)) and three randomization settings per condition.\footnote{The published IEEE camera-ready version states seven objects. This reflects a miscount, not a difference in what \Cref{tab:main_results} reports. Development runs additionally comprised other geometric primitives and YCB objects \cite{calli2015ycb}, but the results tables average over the two object primitives above.} For hardware experiments, we evaluate object pose perturbations following the protocol in \Cref{ssec:hwexp}. \Cref{fig:hwsetup} shows the real robotic manipulator and sensing setup.
\begin{table}[t]
\centering
\caption{Dry-contact friction coefficients (\(\mu_{\mathrm{f}}=1.0\)).}
\label{tab:muexps}
\small
\begin{tabular}{lccc}
\toprule
Object & \(\mu_{\mathrm{o}}\) Range & Nom.\ \(\mu_{\mathrm{o}}\) & \(\mu_{\mathrm{eff}}\) Range \\
\midrule
PLA primitive & 0.30--0.50 & 0.40 & 0.32--0.42 \\
LDPE bottle & 0.25--0.45 & 0.35 & 0.30--0.40 \\
Alum./steel can & 0.20--0.35 & 0.25 & 0.26--0.35 \\
Tennis ball (felt) & 0.50--0.80 & 0.65 & 0.42--0.53 \\
\bottomrule
\end{tabular}
\end{table}
\smallskip

\noindent\textbf{Simulation Stress Tests:}
In simulation, we subject each grasp to a two-phase stress test that evaluates lift stability and disturbance rejection. The robot first lifts the object vertically by 5\,cm over 1\,s to verify load-bearing capability. It then applies four lateral shear pulses of increasing magnitude (3.0, 5.0, 8.0, and 12.0\,N) to test resistance to disturbance-induced slip. We consider a grasp \emph{nominally successful} if the object remains within 1\,cm of the lifted pose and survives all shear pulses. We then evaluate grasp stability, robustness, belief convergence, and empirical failure rate:
\begin{equation}\label{eq:pert_fail}
\hat{P}_{\mathrm{fail}}^{\mathrm{emp}} = 1 - \frac{n_{\mathrm{surv}}}{n_{\mathrm{test}}},
\end{equation}
where \(n_{\mathrm{surv}}\) is the number of grasps surviving all perturbations out of \(n_{\mathrm{test}}\) total trials. We compare this with the belief-predicted failure probability~\eqref{eq:fail_mc}, \(\hat{P}_{\mathrm{fail}}^{\mathrm{bel}} \approx P(C>\tau_f)\), using a Gaussian CDF with \(\tau_f=5.8\) (dimensionless).

\noindent\textbf{Hardware Experiments:}\label{ssec:hwexp}
We evaluate grasp stability under object pose uncertainty (ICP pose estimates) and structured pose perturbations. The protocol uses a \(24\times19\) pegboard (1-inch / 25.4\,mm pitch) with 10\,mm dots rigidly fixed to the workspace, with the hand positioned near the board center so all perturbations lie within a reachable grid. For each object we define an \(M\times M\) sub-grid centered on the nominal pose, where \(M=2N_o^j+1\) (odd \(N_o^j\)) or \(M=2N_o^j\) (even \(N_o^j\)), with \(N_o^j=\lceil s/25.4\rceil\) and \(s\) the support-side length in mm. We evaluate four offsets
\begin{equation}\label{eq:hw_offs}
\begin{array}{l}
\mathcal{O}=\{(0,0),(+1,0),(0,+1),(+1,+1)\}\times5\,\text{mm},
\end{array}
\end{equation}
corresponding to the center, half-dot translations in \(+x\) and \(+y\), and a \(45^\circ\) diagonal (\(\approx7.07\)\,mm). At each offset the ICP pipeline records the object pose \(X^o\), from which we compute the error \(\Delta X=(\hat{X}^o_{\mathrm{ref}})^{-1}X^o\). The robot then lifts the object by 4\,cm over 3\,s and holds for 2\,s while the Intel D435i monitors displacement. A trial succeeds if the object remains within 2\,cm of the lifted reference during the hold phase. With three objects (two YCB and one primitive), four offsets, and two methods, \vnb\ and a Gaussian baseline (\gauss), our hardware evaluation comprises 24 trials total (12 per method). \Cref{apx:experiments} provides the ICP scoring and pegboard-placement details.
\begin{table}[t]
\centering
\small
\caption{Grasp Stress test protocol and evaluation metrics}
\label{tab:metsum}
\renewcommand{\arraystretch}{1}
\begin{tabular}{L{2cm} p{5cm}}
\toprule
\toprule
Metric Category & Definition \\
\midrule
\textbf{Nominal Stress Test (Sim.)}
& Vertical lift of 5\,cm over 1\,s, followed by lateral shear pulses of 3, 5, 8, and 12\,N. Success requires displacement \(< 1\)\,cm from the lifted pose. \\
\midrule
\textbf{Uncertain Friction Stress Test (Sim.)}
& 28 perturbations: lateral impulses (3, 5, 8, 12\,N; 4 directions; 0.15\,s), torque impulses (0.3--1.0\,Nm; 3 axes; 0.2\,s), and friction drops \(\mu \in \{0.05, 0.10, 0.15\}\). \\
\midrule
\textbf{Per-Episode Stability}
& Lift height \(h_{\mathrm{lift}}\), time-to-slip \(t_{\mathrm{slip}}\), peak slip \(d_{\mathrm{slip}}^{\max}\), and failure mode (\texttt{grasp}, \texttt{lift}, \texttt{slip}, \texttt{perturbation}, \texttt{none}). \\
\midrule
\textbf{Robust Success Criterion}
& Nominal success and perturbation survival rate \(> 50\%\), where survival rate is the fraction of perturbations completed without exceeding displacement or rotation thresholds. \\
\midrule
\textbf{Hardware Evaluation}
& Approximate \(\varepsilon\) quality at grasp acquisition (\(\hat{\varepsilon}\)), mean slip proxy (\(\bar{s}\)). \\
\bottomrule
\bottomrule
\end{tabular}
\end{table}

\begin{table*}[t]
\centering
\caption{Aggregate performance across friction regimes (mean over objects, \(\beta\) values, and seeds) in simulation. Robust = nominally successful \& perturbation survival \(\geq 50\%\). PertSurv = perturbation survival rate. \(\varepsilon\) = mean Ferrari--Canny quality metric~\cite{ferrari1992planning} at episode termination. Quality = mean fraction of episodes achieving positive force closure (\(\varepsilon>0\)) at termination.
\(\hat{P}_{\mathrm{fail}}^{\mathrm{bel}}\) = belief-predicted failure probability. \(\hat{P}_{\mathrm{fail}}^{\mathrm{emp}}\) = empirical failure probability~\eqref{eq:pert_fail}. Best per regime in bold; \(|\Delta\hat{P}|\) = calibration error \(|\hat{P}_{\mathrm{fail}}^{\mathrm{bel}} - \hat{P}_{\mathrm{fail}}^{\mathrm{emp}}|\).}
\label{tab:main_results}
\small
\renewcommand{\arraystretch}{0.85}
\setlength{\tabcolsep}{3.5pt}
\begin{tabular}{@{}l l c c c c c c c c c@{}}
\toprule
\textbf{Method} & \textbf{Regime} & \textbf{SR} (\%) & \textbf{Robust\,\%} & \textbf{PertSurv\,\%} & \(\varepsilon\uparrow\) & \textbf{Quality} \(\uparrow\) & \(\hat{P}_{\mathrm{fail}}^{\mathrm{bel}}\downarrow\) & \(\hat{P}_{\mathrm{fail}}^{\mathrm{emp}}\downarrow\) & \(|\Delta\hat{P}|\) & \textbf{Time}\,(s) \\
\midrule
\multirow{3}{*}{\small\pf}
  & nominal      & 83  & 67 & 57 & 0.0058 & 0.86 & {--} & 0.43 & {--} & 49.6 \\
  & wide         & 44  & 39 & 31 & 0.0039 & 0.85 & {--} & 0.69 & {--} & 32.6 \\
  & bimodal      & 56  & 28 & 29 & 0.0069 & 0.87 & {--} & 0.71 & {--} & 41.7 \\
\midrule
\multirow{3}{*}{\small\gauss}
  & nominal      & \textbf{100} & \textbf{75} & \textbf{69} & 0.0104 & 0.90 & 0.22 & \textbf{0.31} & 0.09 & \textbf{3.2} \\
  & wide         & 71  & 63 & 56 & 0.0100 & 0.90 & 0.31 & 0.44 & 0.13 & 3.4 \\
  & bimodal      & 54  & 42 & 36 & 0.0111 & 0.90 & 0.57 & 0.64 & \textbf{0.07} & 4.3 \\
\midrule
\multirow{3}{*}{\small\gausscvar}
  & nominal      & \textbf{100} & \textbf{75} & \textbf{69} & 0.0104 & 0.90 & 0.22 & \textbf{0.31} & 0.09 & 3.1 \\
  & wide         & 71  & 63 & 56 & 0.0100 & 0.90 & 0.31 & 0.44 & 0.13 & 3.4 \\
  & bimodal      & 54  & 42 & 36 & 0.0111 & 0.90 & 0.57 & 0.64 & \textbf{0.07} & 4.3 \\
\midrule
\multirow{3}{*}{\small\cem}
  & nominal      & \textbf{100} & 63 & 55 & 0.0101 & 0.90 & 1.00 & 0.45 & 0.55 & 5.5 \\
  & wide         & \textbf{79}  & \textbf{67} & \textbf{59} & 0.0109 & 0.90 & 1.00 & \textbf{0.42} & 0.58 & 5.4 \\
  & bimodal      & 42  & 33 & 30 & 0.0119 & \textbf{0.91} & 1.00 & 0.70 & 0.29 & 5.1 \\
\midrule
\addlinespace[2pt]
\multirow{3}{*}{\textbf{Ours (\vnb)}}
  & nominal       & \textbf{100}  & \textbf{79}   & \textbf{73} & \textbf{0.0105} & 0.90 & \textbf{0.18}   & \textbf{0.27} & \textbf{0.09} & 8.5 \\
  & wide          & \textbf{79}   & \textbf{67}   & 58 & \textbf{0.0106} & 0.90 & \textbf{0.28}   & \textbf{0.42} & \textbf{0.14} & 7.1 \\
  & bimodal       & \textbf{54}   & \textbf{38}   & \textbf{35} & 0.0106 & 0.89 & \textbf{0.54}   & \textbf{0.65} & \textbf{0.11} & 6.8 \\
\bottomrule
\end{tabular}
\end{table*}

\section{Results \& Discussion}\label{sec:resdis}
\begin{figure}[t]
    \centering
    \includegraphics[width=\linewidth, trim=0pt 10pt 0pt 0pt, clip]{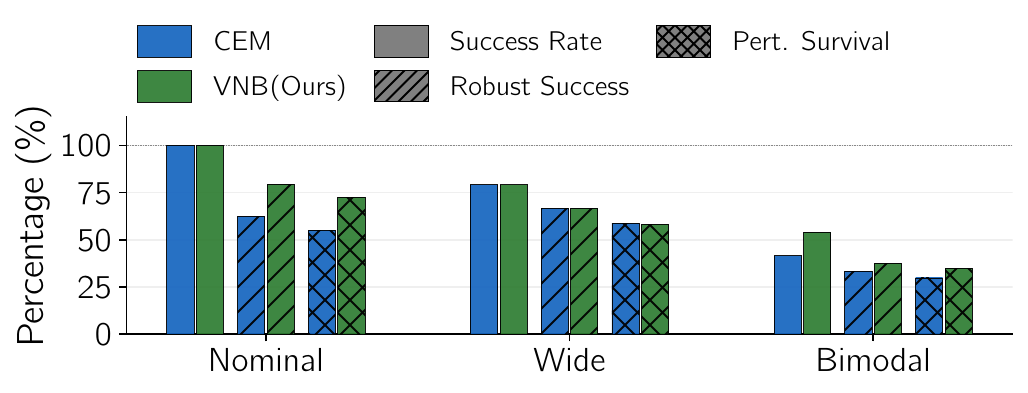}
    \caption{\textbf{Per-Regime Performance Comparison.} Grouped bar chart comparing {\small\cem} and {\small\vnb} across friction regimes. \vnb~matches or exceeds {\small\cem} in all three regimes, with the largest gains under nominal and bimodal friction.}
    \label{fig:regime_bars}
\end{figure}%

\subsection{Simulation Results}
\noindent\Cref{tab:main_results} compares all five methods across friction regimes, averaged over objects, risk levels, and seeds (time reports wall-clock seconds per episode). While nominal success is high for several methods, the perturbation protocol separates policies by whether nominal grasps remain stable under uncertain friction and external loading. We therefore emphasize success rate (SR), robust success, perturbation survival, and calibration error rather than raw worst-tail quality: once a grasp fails, terminal quality collapses to zero and no longer distinguishes whether a method failed often or rarely. \Cref{fig:regime_bars} and \Cref{fig:epstraj} show per-regime stress-test performance and grasp-quality evolution against \cem. \gauss\ and \gausscvar\ produce identical aggregate results across all regimes in our benchmark: the single-component Gaussian belief (\(K=1\)) cannot represent distinct contact hypotheses, and the CVaR weighting did not change the selected actions after multi-start optimization.

\begin{figure}[t]
    \centering
    \includegraphics[width=\linewidth]{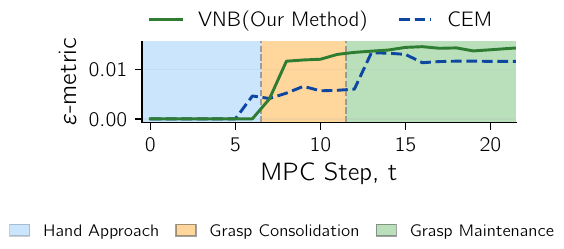}
    \caption{\textbf{Grasp Quality under Force Perturbations.} Aggregate Ferrari--Canny \(\varepsilon\) over MPC steps for {\small\vnb} and {\small\cem}. Shaded regions mark approach, consolidation, and grasp maintenance. {\small\vnb} maintains higher \(\varepsilon\) and remains robust under lift perturbations (step 11), whereas expectation-driven {\small\cem} degrades during grasp maintenance.}
    \label{fig:epstraj}
\end{figure}

\begin{figure}[htb]
    \centering
    \setlength{\tabcolsep}{0pt}
    \renewcommand{\arraystretch}{0.8}
    \setlength{\fboxsep}{0pt}
    \setlength{\fboxrule}{1pt}

    \begin{tabular}{@{}cc ccc@{}}
    & \small\textbf{Pre-Grasp} & \small\textbf{Grasp} & \small\textbf{Lift (Mustard)}  &  \small\textbf{Lift (Box)}  \\[2pt]

    \rotatebox{90}{\small\;\textbf{\cem}} &
    \fbox{\includegraphics[width=0.23\linewidth, trim=970pt 700pt 730pt 130pt, clip]{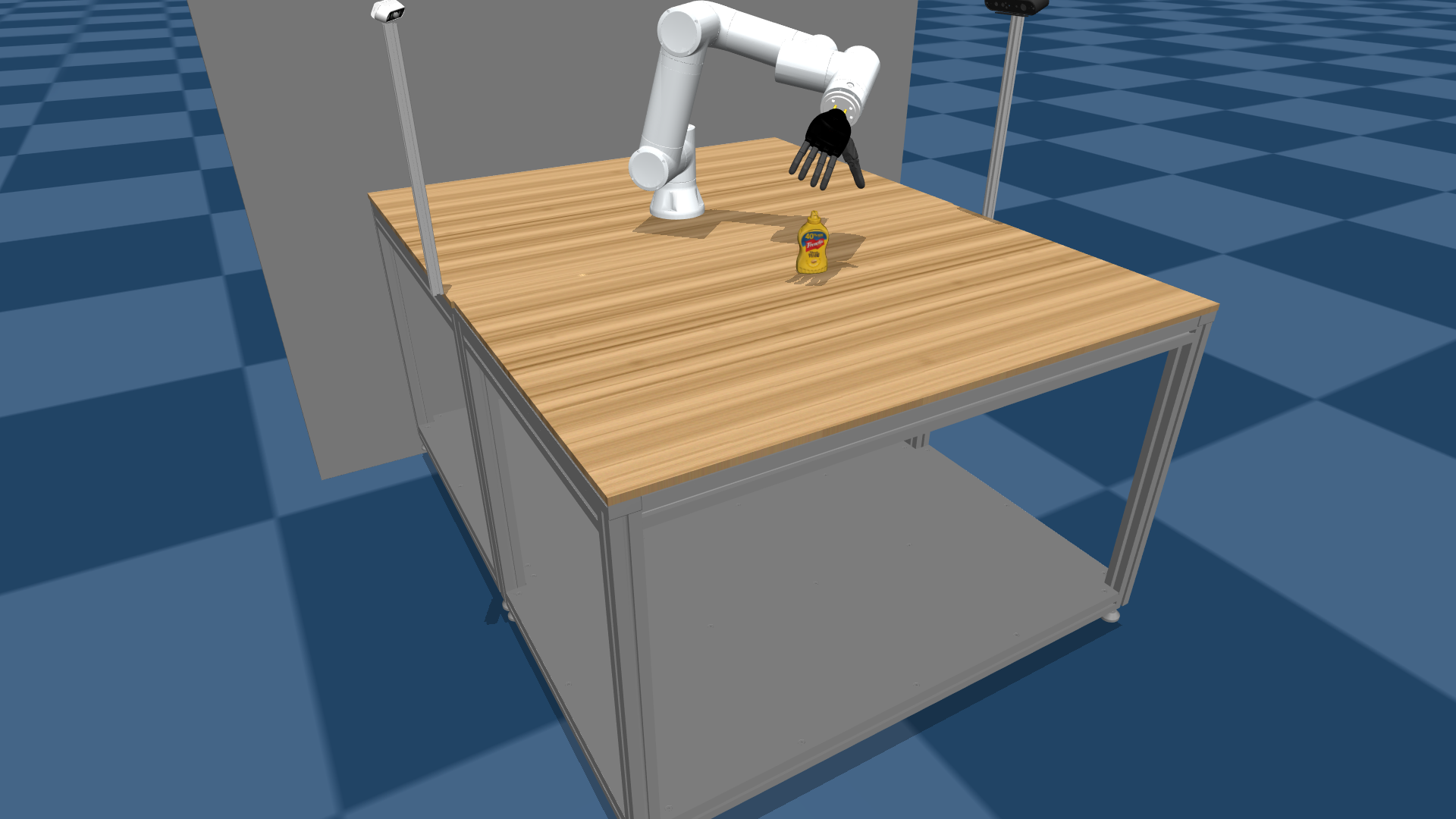}} &
    \fbox{\includegraphics[width=0.23\linewidth, trim=970pt 700pt 730pt 130pt, clip]{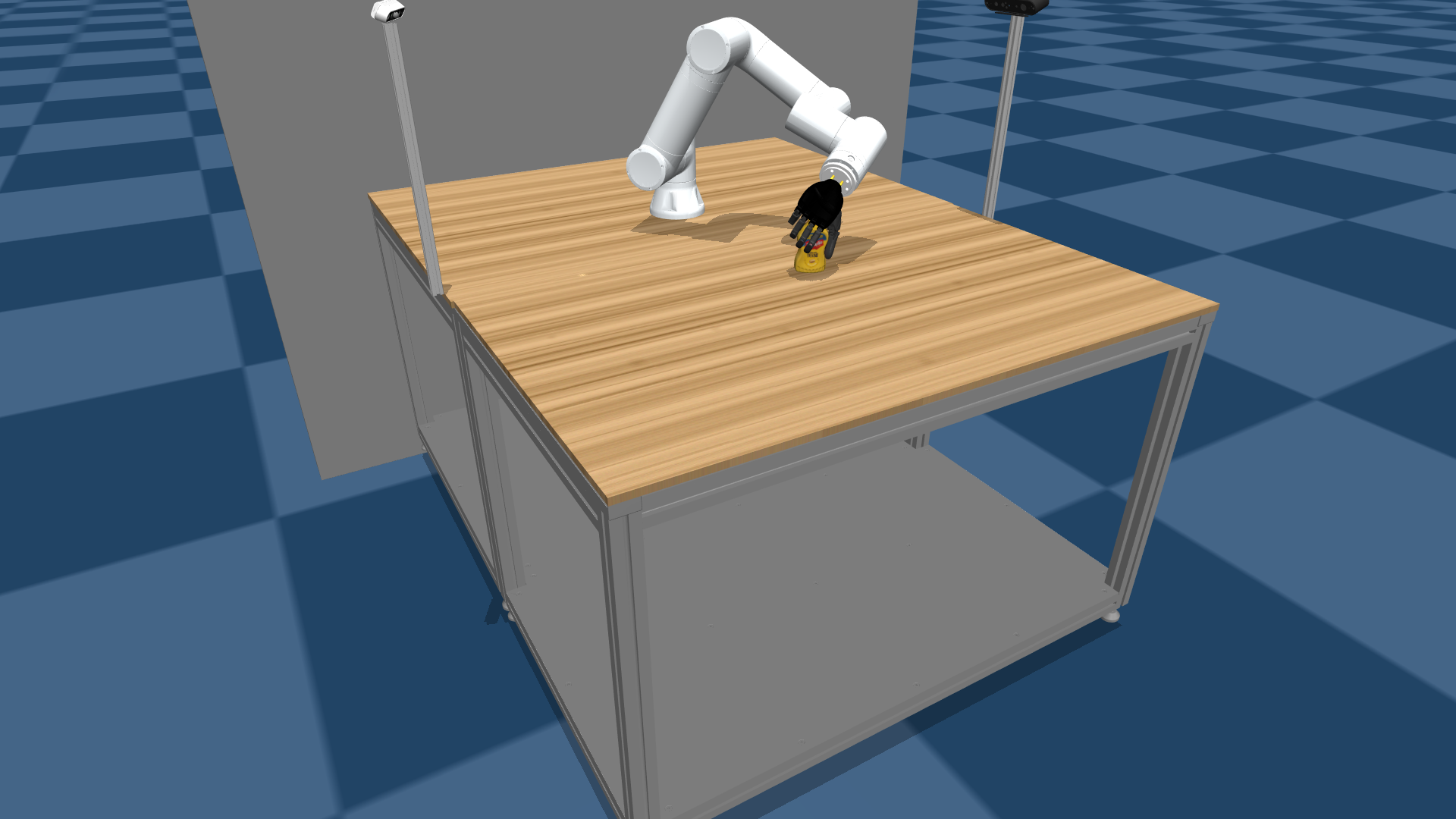}} &
    \fbox{\includegraphics[width=0.23\linewidth, trim=970pt 700pt 730pt 130pt, clip]{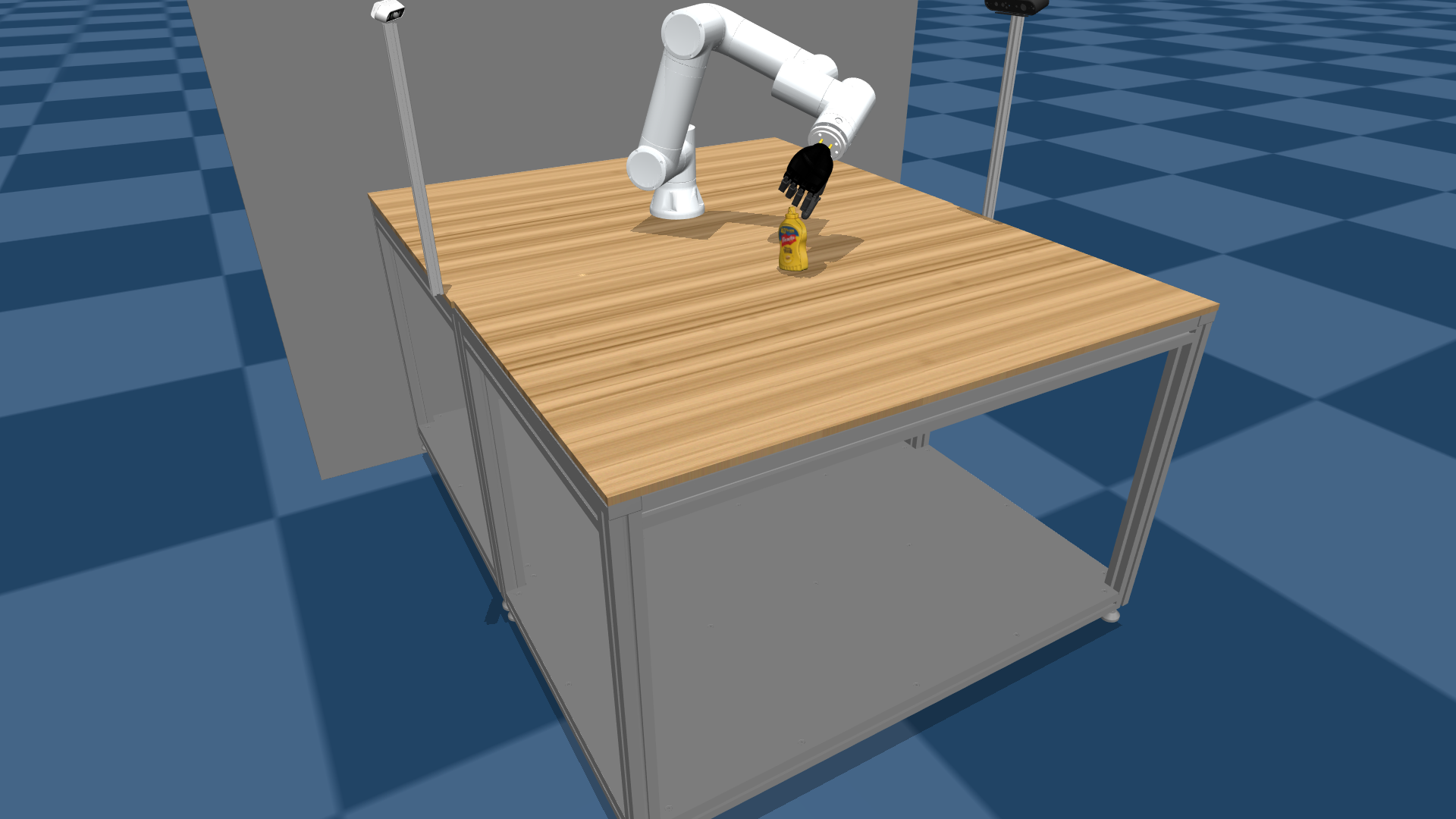}} & \fbox{\includegraphics[width=0.20\linewidth, trim=970pt 700pt 760pt 130pt, clip]{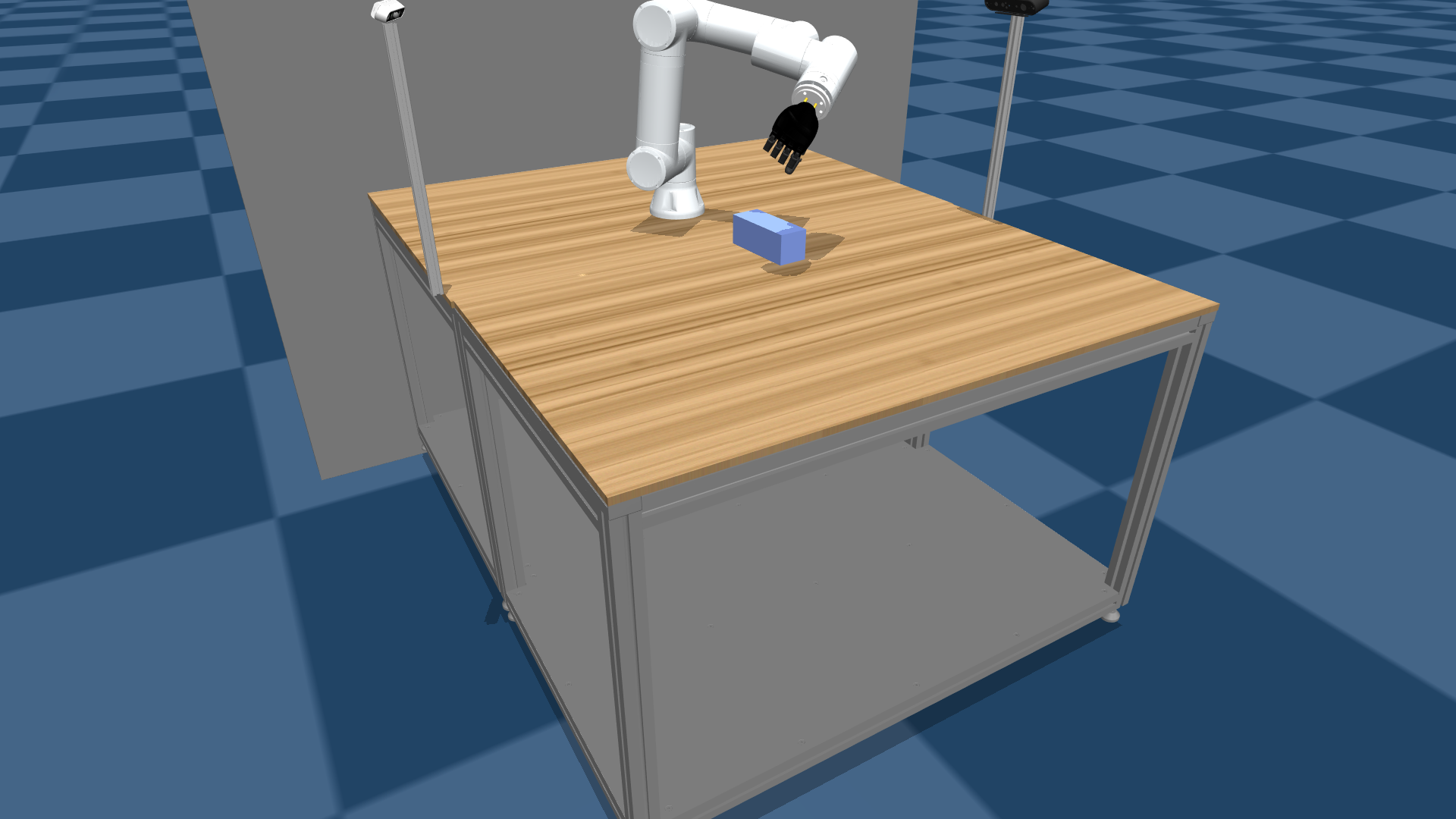}} \\[-1pt]

    & \scriptsize \(\varepsilon{<}0\) &
      \scriptsize \(\varepsilon{=}0.001\) &
      \scriptsize \textcolor{red}{\(\varepsilon{=}0.001\)\;(failure)} & \scriptsize \textcolor{red}{\(\varepsilon{=}0.0022\)\;(failure)}\\[4pt]

    \rotatebox{90}{\small\;\textbf{\vnb}} &
    \fbox{\includegraphics[width=0.23\linewidth, trim=970pt 700pt 730pt 130pt, clip]{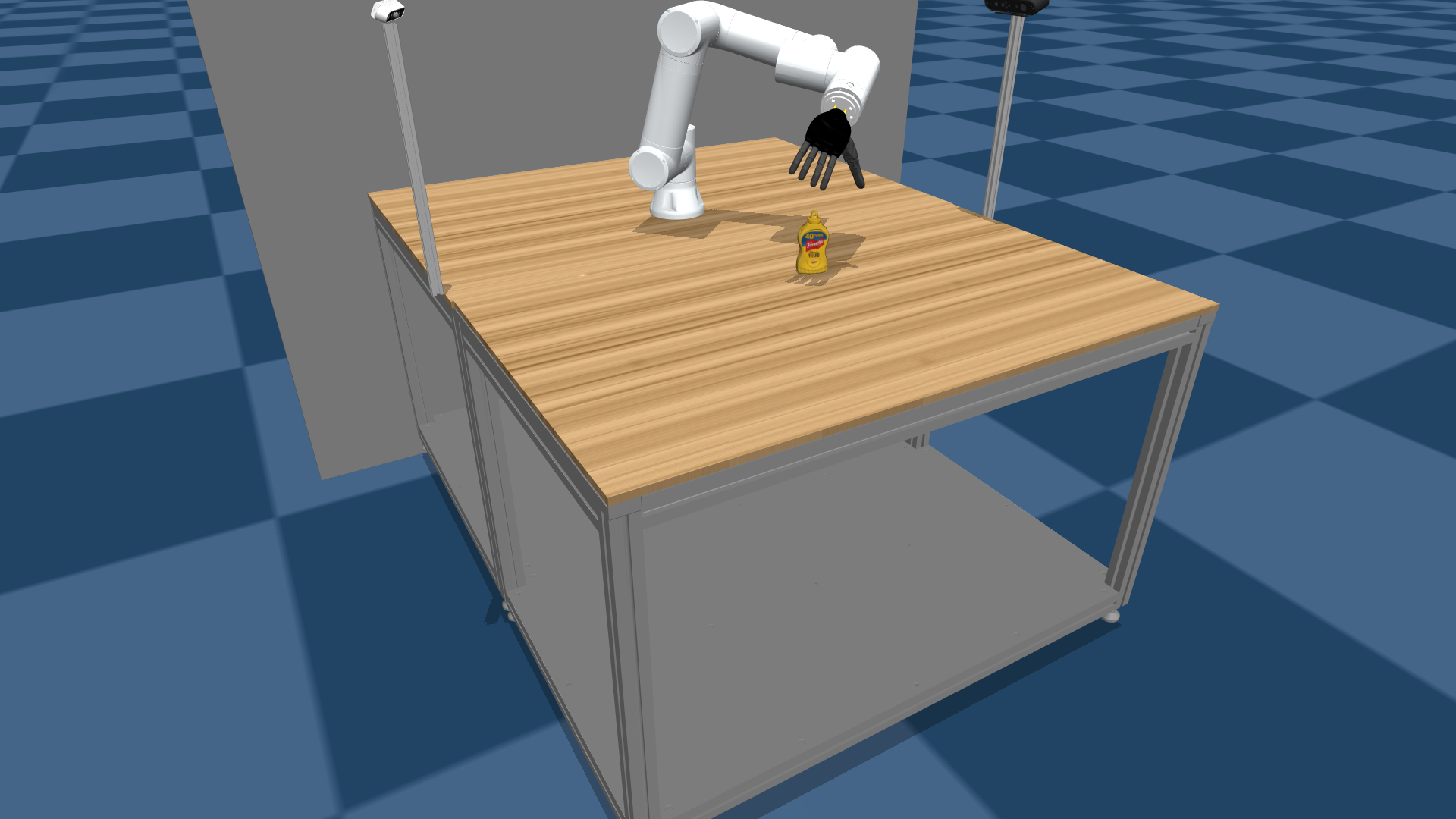}} &
    \fbox{\includegraphics[width=0.23\linewidth, trim=970pt 700pt 730pt 130pt, clip]{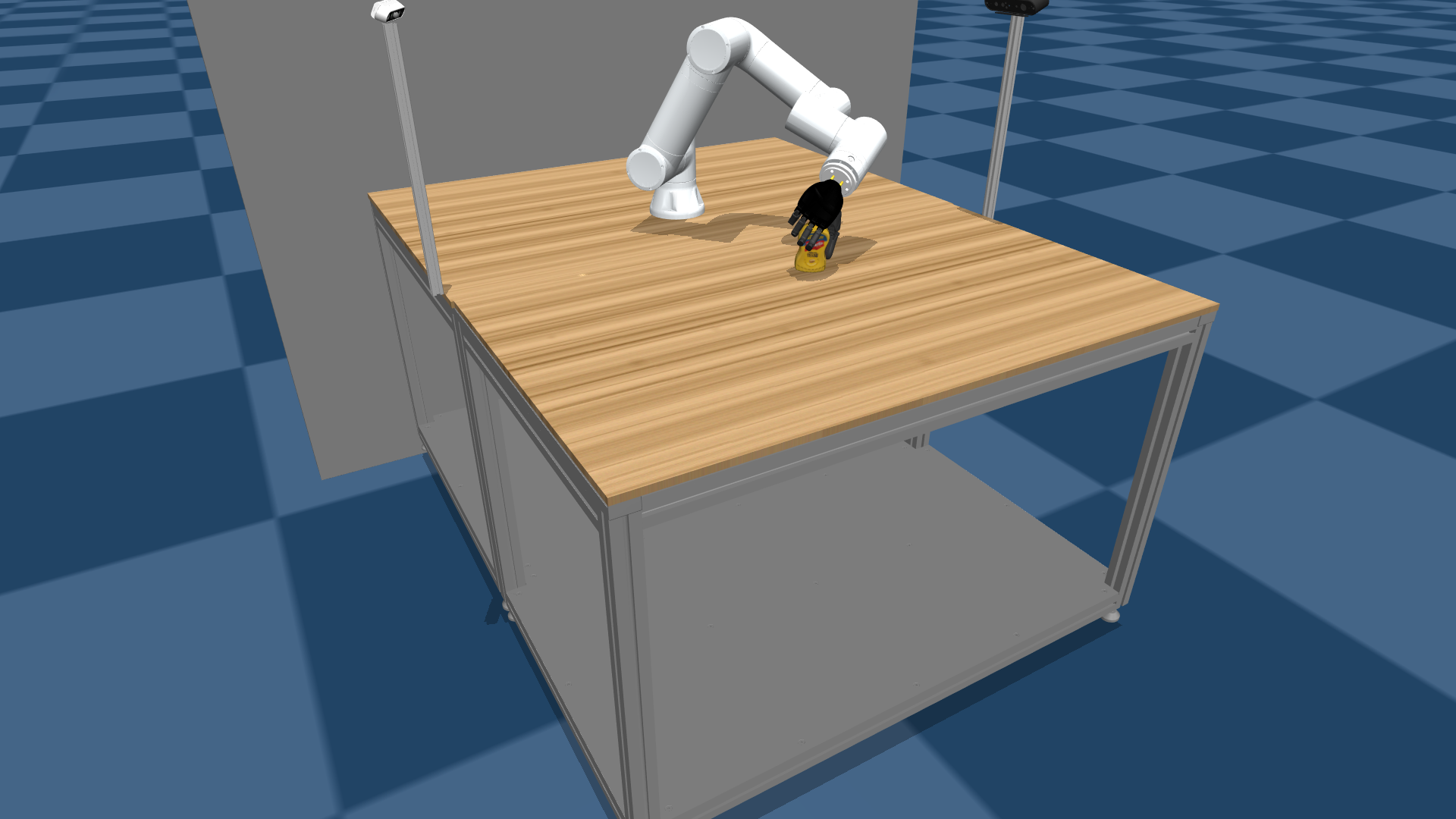}} &
    \fbox{\includegraphics[width=0.23\linewidth, trim=970pt 700pt 730pt 130pt, clip]{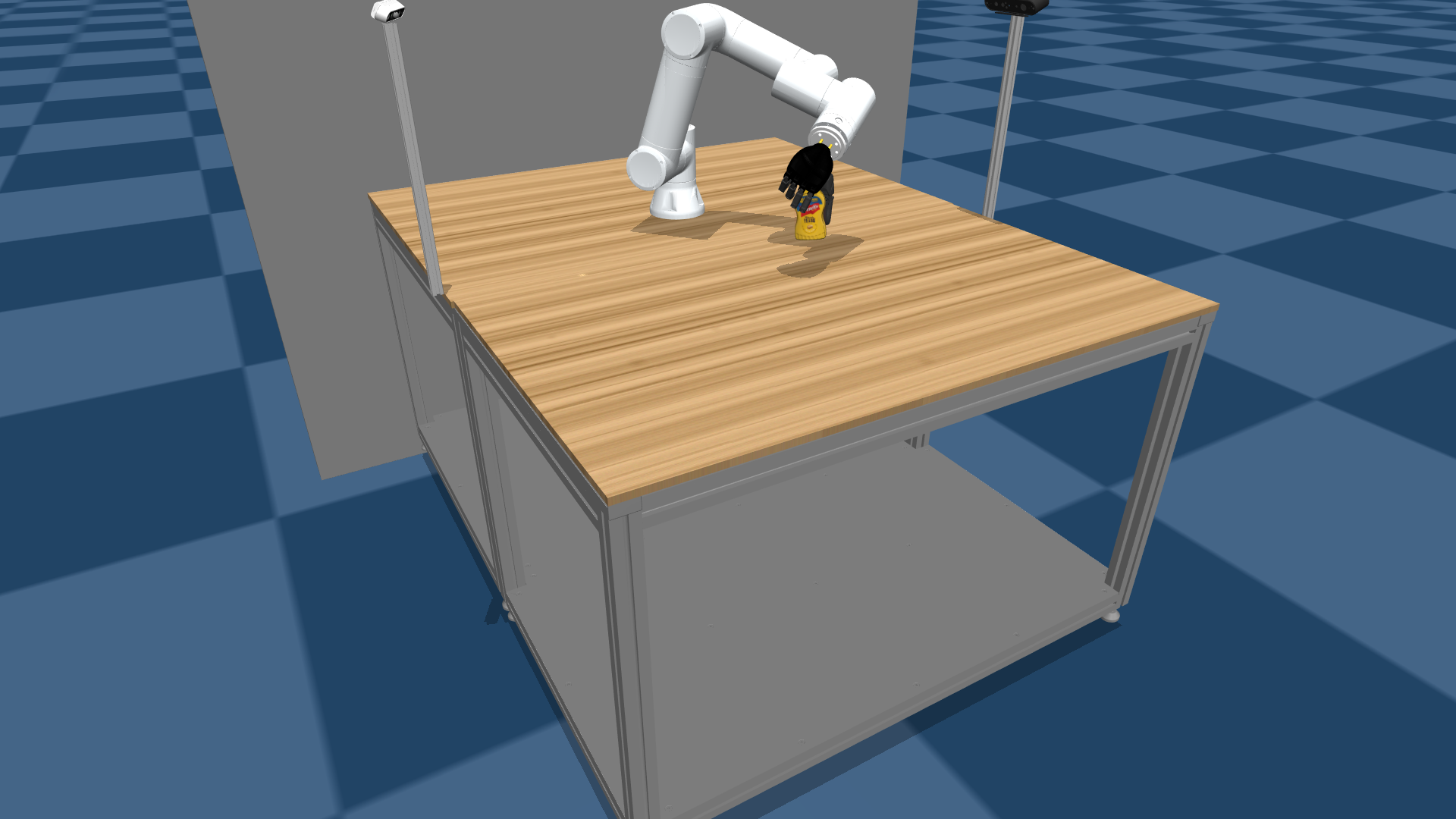}} & \fbox{\includegraphics[width=0.20\linewidth, trim=970pt 700pt 760pt 130pt, clip]{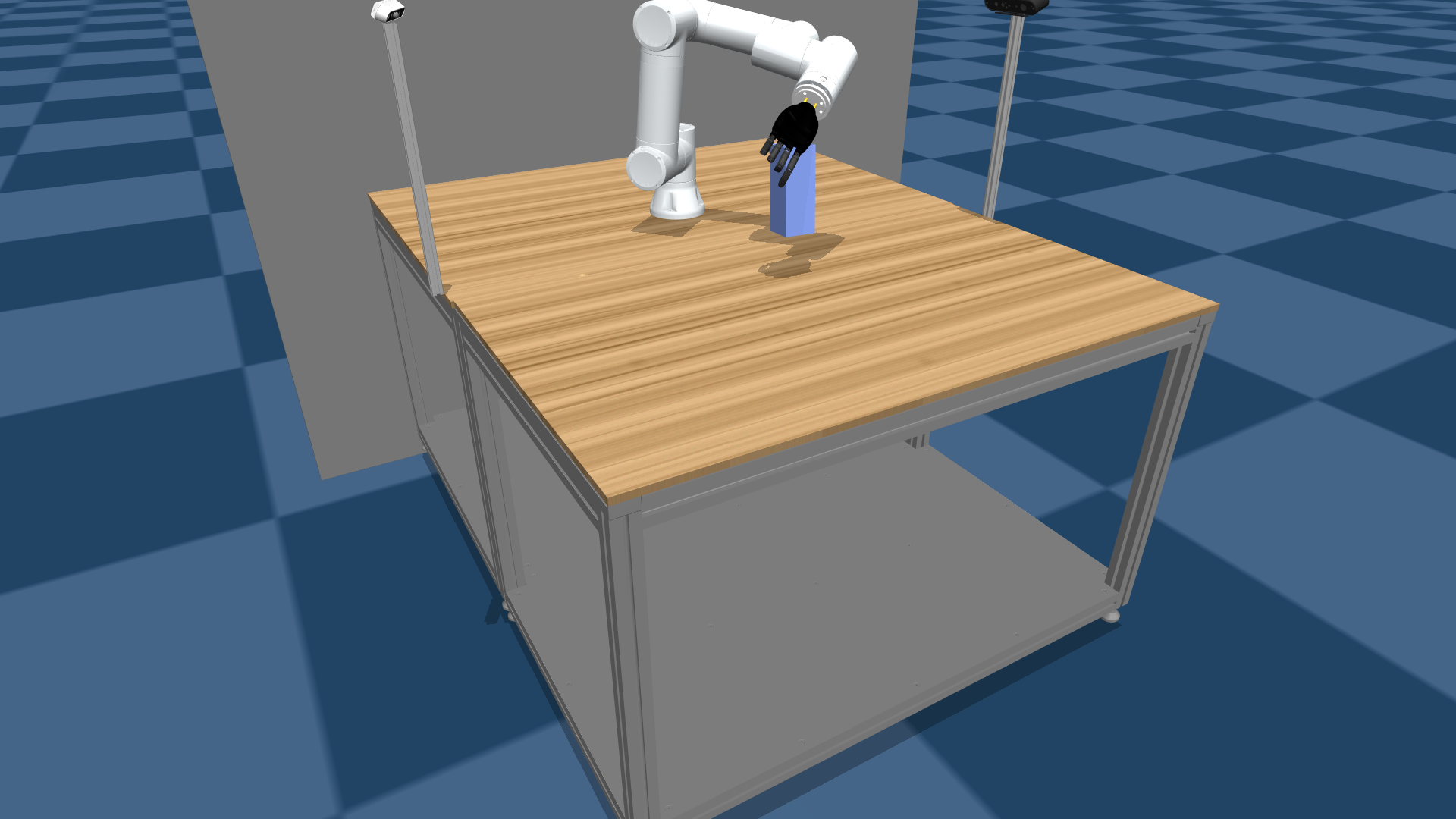}} \\[-1pt]

    & \scriptsize \(\varepsilon{<}0\) &
      \scriptsize \(\varepsilon{=}0.03\) &
      \scriptsize {\color{blue!85!black}{\(\varepsilon{=}0.15\)\;(success)}} &
      \scriptsize {\color{blue!85!black}{\(\varepsilon{=}0.0021\)\;(success)}} \\
    \end{tabular}

    \caption{\textbf{Grasp Comparison on Representative Objects.} \emph{Top}: Under friction uncertainty, {\small\cem} fails to achieve force closure (\(\varepsilon{\approx}0\), \(0.0022\)) and drops the object during lift. \emph{Bottom}: Our risk-aware {\small\vnb} achieves robust grasps (\(\varepsilon{=}0.15\), \(0.0021\)) and sustains a 50\,mm lift.}
    \label{fig:lifttest}
\end{figure}%

\subsection{Hardware Results}%
\noindent\Cref{tab:hw_results} summarizes our hardware evaluation on three YCB objects (box, mustard bottle, soup can; see \Cref{fig:hwsetup,fig:hwgrasps}) under pose perturbations across 4 offsets per object (\(n=12\) trials total). Both methods achieve 100\% success on this small hardware set, so the hardware trials should be interpreted as an execution and convergence check rather than a statistically powered success-rate comparison. Within this setting, \vnb\ uses fewer steps and less runtime, terminating through force-threshold or tactile-quality stopping conditions across all trials. \gauss\ has a heavier upper tail in step count, with an interquartile range (IQR) extending to 11 steps, indicating less consistent convergence. The median step count is 6 for \vnb\ and 7 for \gauss; the median runtime is 11.5\,s for \vnb\ and 14.2\,s for \gauss.

\vnb\ also has lower relative dispersion: CoV is 0.32 vs.\ 0.49 for step count and 0.41 vs.\ 0.78 for runtime. Terminal grasp quality~\(\hat{\varepsilon}_T\) is higher for \vnb\ (median \(1.6\times10^{-3}\) vs.\ \(0.9\times10^{-3}\)), with lower variance (CoV 0.48 vs.\ 1.21). The wider IQR of \gauss\ (\(0.1\)--\(2.2\)) reflects less consistent terminal tactile quality and suggests occasional lower-quality terminations. Both methods achieve a median of three final contacts. On the mustard bottle, peak and mean slip are comparable across methods (63.2 vs.\ 64.1 and 3.2 vs.\ 3.3, respectively), while \gauss\ has tighter slip dispersion on this object (CoV 0.13 vs.\ 0.42). The small sample size (\(n=12\) per method) limits statistical power; we report medians and IQRs rather than means and standard deviations.%
\begin{figure}[htb]
\centering
\setlength{\tabcolsep}{2pt}
\renewcommand{\arraystretch}{0.85}
\setlength{\fboxsep}{0pt}
\setlength{\fboxrule}{1pt}

\newlength{\hwimgw}
\setlength{\hwimgw}{0.29\linewidth}

\begin{tabular}{@{}cccc@{}}
\small\textbf{Mustard Bottle} & \small\textbf{Box} & \small\textbf{Soup Can} \\[3pt]
\fbox{\includegraphics[width=\hwimgw,height=\hwimgw,keepaspectratio,
    trim=300pt 20pt 250pt 500pt,clip]
    {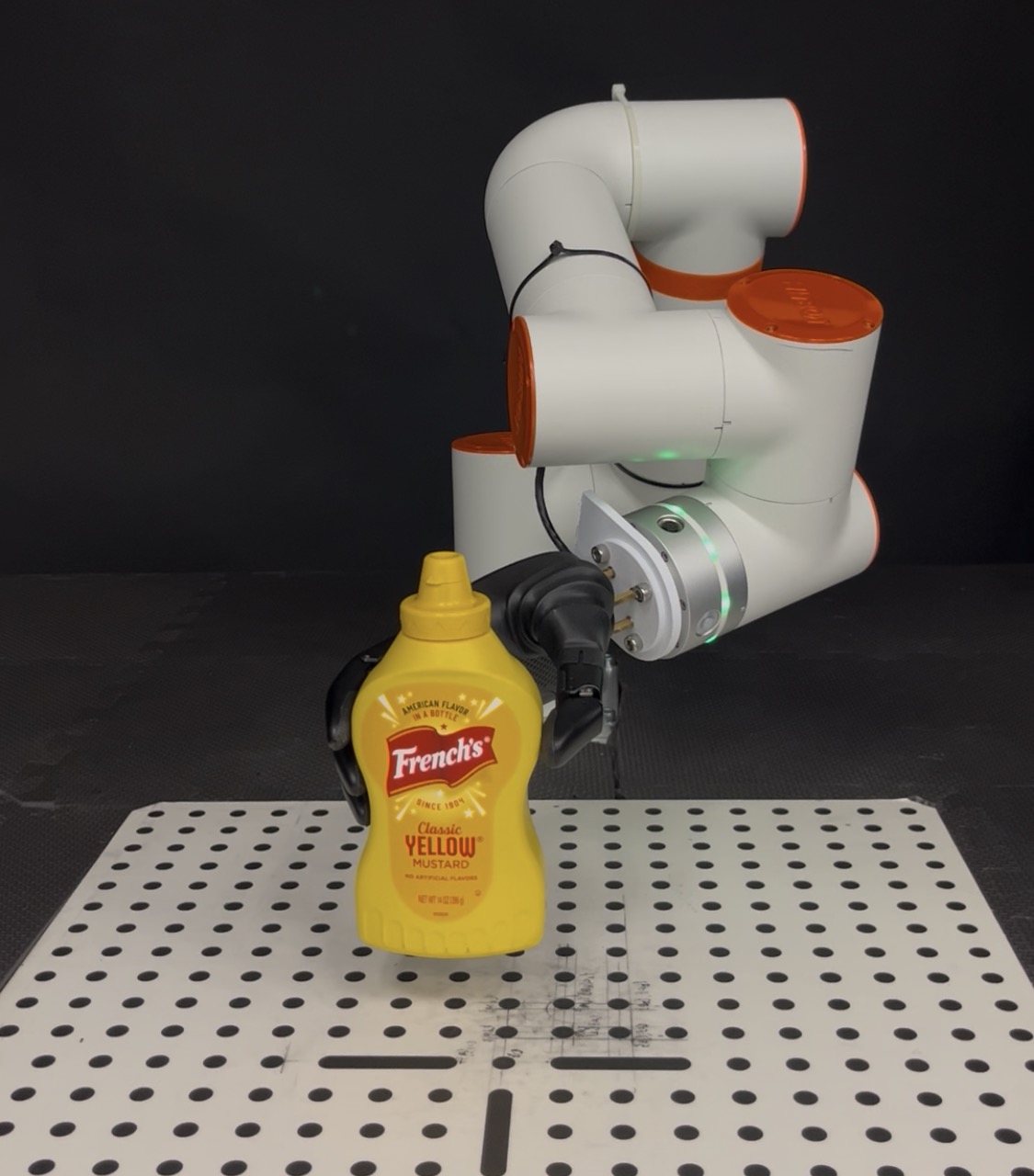}} &
\fbox{\includegraphics[width=\hwimgw,height=\hwimgw,keepaspectratio,
    trim=300pt 20pt 200pt 300pt,clip]
    {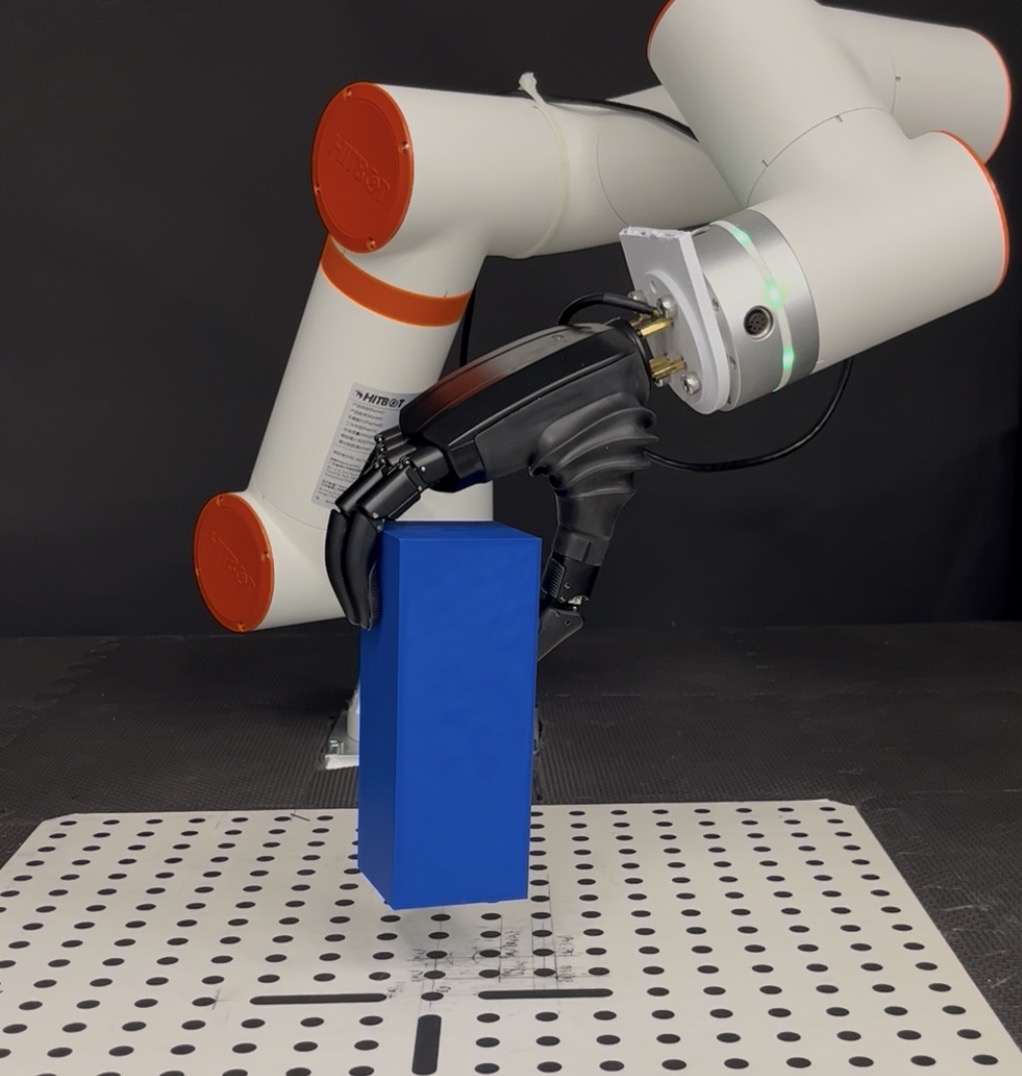}} &
\fbox{\includegraphics[width=\hwimgw,height=\hwimgw,keepaspectratio,
    trim=100pt 20pt 200pt 300pt,clip]
    {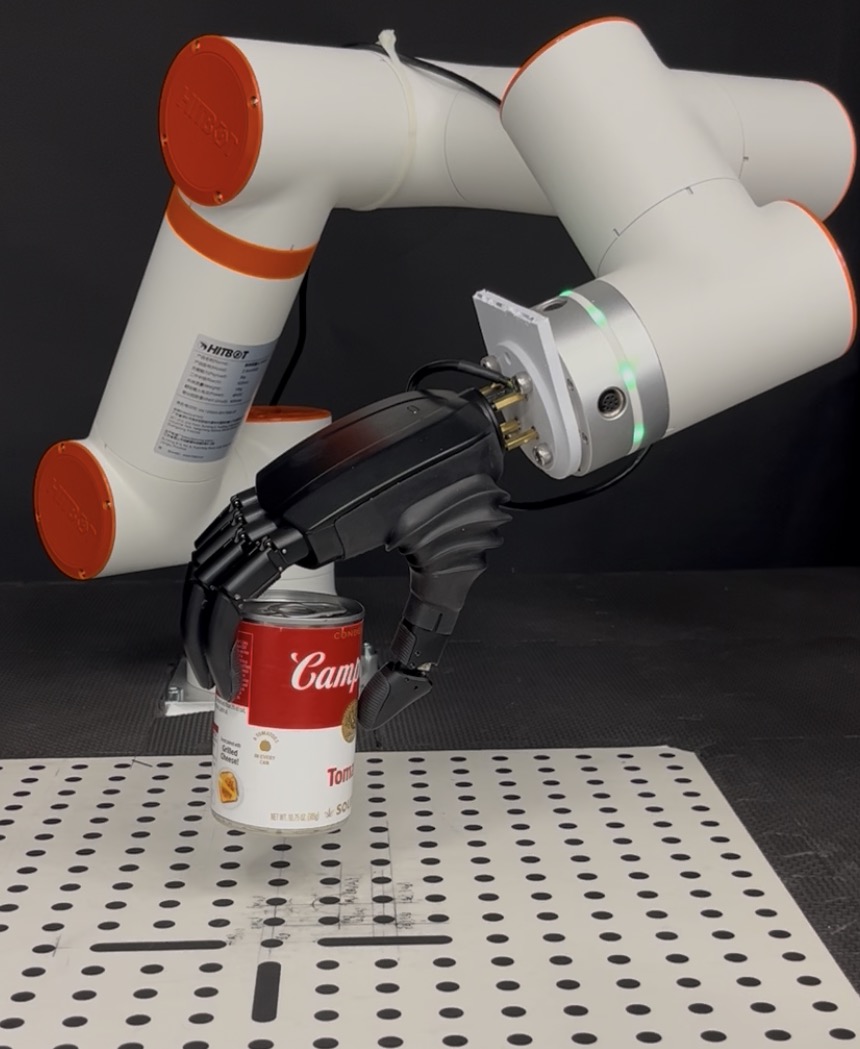}} \\[-2pt]
{\small\(\hat{\varepsilon}_T = 1.539\)} &
{\small\(\hat{\varepsilon}_T = 1.637\)} &
{\small\(\hat{\varepsilon}_T = 3.066\)} \\
\end{tabular}

\caption{\textbf{Hardware Grasp Results under Vision Uncertainty.} Post-lift frames of representative \vnb{} grasps executed on our hardware platform after 6D pose estimation (RealSense D435i). Here, we quantify terminal grasp quality using an \(\varepsilon\)-metric \cite{ferrari1992planning} inspired proxy
\(\hat{\varepsilon}_T\) (\(\times 10^{-3}\)) estimated from piezoresistive tactile observations (see \Cref{apx:ehatcomp} for details).}
\label{fig:hwgrasps}
\end{figure}

\subsection{Summary}%
\noindent\vspace{-0.1em}
Together, our simulation and hardware results demonstrate that continuous, reparameterizable beliefs close the gradient gap between the planning objective and the belief state that particle filters leave open. The resulting pathwise gradients through the smooth CVaR surrogate improve robust success, perturbation survival, and calibration while retaining substantially lower wall-clock cost than particle-filter MPC. Our proposed composite objective~\eqref{eq:mpcobj} integrates contact risk, visual perception cost, and information-seeking entropy, enabling the grasp planner to hedge against multimodal uncertainty while avoiding hand configurations that occlude the object. Furthermore, our perturbation protocol across simulation and hardware experiments provides an additional evaluation mechanism beyond aggregate success rates: by comparing the belief-predicted failure probability~\eqref{eq:fail_mc} against the empirical perturbation failure rate~\eqref{eq:pert_fail}, we can assess belief calibration (see the \(\Delta \hat{P}\) column of \Cref{tab:main_results}). A well-calibrated belief should satisfy \(\hat{P}_{\mathrm{fail}}^{\mathrm{bel}} \approx \hat{P}_{\mathrm{fail}}^{\mathrm{emp}}\), whereas overconfident beliefs underestimate failure mass and yield fragile grasps despite nominal success. Our results therefore indicate our method's advantage in producing better-calibrated uncertainty estimates: across all regimes, \vnb\ consistently achieves small calibration errors (\(|\Delta\hat{P}| \le 0.14\)) while maintaining high success and robustness. In contrast, the sampling-based planner baseline (\cem) miscalibrates risk, predicting near-certain failure (\(\hat{P}_{\mathrm{fail}}^{\mathrm{bel}} \approx 1\)) even when empirical failure rates are moderate, indicating beliefs that fail to reflect actual grasp stability.
\begin{table}[!t]
\centering
\setlength{\abovecaptionskip}{2pt}
\setlength{\belowcaptionskip}{2pt}
\caption{Hardware evaluation over 12 trials per method and three objects. Both methods achieve 100\% success on this small set; \vnb\ converges faster and with lower dispersion. Continuous metrics are median [IQR]; CoV is \(\sigma/\mu\).}
\label{tab:hw_results}
\scriptsize
\setlength{\tabcolsep}{2pt}
\renewcommand{\arraystretch}{0.76}
\begin{tabular}{@{}l cc cc@{}}
\toprule
& \multicolumn{2}{c}{\textbf{\vnb}} & \multicolumn{2}{c}{\textbf{\gauss}} \\
\cmidrule(lr){2-3}\cmidrule(lr){4-5}
\textbf{Metric}
  & Med. [IQR] & CoV \(\downarrow\)
  & Med. [IQR] & CoV \(\downarrow\) \\
\midrule
Steps
  & \textbf{6 {[}6--7{]}}   & \textbf{0.32}
  & 7 {[}6--11{]}           & 0.49 \\
Time (s)
  & \textbf{11.5 {[}7.5--14.5{]}} & \textbf{0.41}
  & 14.2 {[}10.3--15.8{]}         & 0.78 \\
\(\hat{\varepsilon}_T\) (\(\times10^{-3}\))
  & \textbf{1.6 {[}1.5--2.1{]}} & \textbf{0.48}
  & 0.9 {[}0.1--2.2{]}          & 1.21 \\
Contacts & \multicolumn{2}{c}{\textbf{3}} & \multicolumn{2}{c}{3} \\
Max Slip\footnotemark[3]
  & \textbf{63.2 {[}51.9--64.9{]}} & {0.42}
  & 64.1 {[}60.1--64.3{]}          & \textbf{0.13} \\
Mean Slip\footnotemark[3]
  & \textbf{3.2 {[}2.6--3.3{]}} & {0.41}
  & 3.3 {[}3.1--3.3{]}          & \textbf{0.13} \\
\bottomrule
\end{tabular}
\\[2pt]
{\scriptsize \textsuperscript{3}Mustard bottle only. Slip values were
comparable across other objects.}
\vspace{-6pt}
\end{table}

\section{Conclusion}\label{sec:conc}
\nopagebreak
We presented variational neural beliefs for risk-aware dexterous grasping. Our framework represents uncertainty in contact parameters and object pose with a differentiable belief, enabling pathwise gradients of a smooth CVaR surrogate and direct gradient-based optimization of grasp strategies. We combine contact risk, perception uncertainty, failure probability, and belief entropy in a composite objective that improves robustness while reducing uncertainty, and we use neural belief dynamics to construct a fully differentiable belief-space planner. Across three friction regimes and a 28-test perturbation protocol, our method achieves higher robust success than sampling-based baselines such as CEM while planning 5--7\(\times\) faster than particle-filter MPC in Table~\ref{tab:main_results}. Future work will replace Langevin sampling with normalizing flows for faster reparameterized sampling, extend the risk objective to spectral and entropic coherent risk measures, meta-learn belief dynamics~(\Cref{sec:bel_dyn}) across object classes to reduce calibration cost, and improve tactile contact detection for compliant, low-friction objects by adapting pressure thresholds to object compliance and surface properties.

\bibliographystyle{ieeetr}
\bibliography{references}

\appendix
\makeatletter
\renewcommand{\p@subsection}{}
\makeatother

\setlength{\textfloatsep}{6pt plus 1pt minus 2pt}
\setlength{\floatsep}{5pt plus 1pt minus 2pt}
\setlength{\intextsep}{5pt plus 1pt minus 2pt}

\subsection{Pathwise Risk Gradients via Action Optimization}\label[appendix]{apx:riskgrad}

\begin{theorem}[Reparameterized CVaR Gradient]\label{thm:cvar_grad}
Let \(b(\boldsymbol{\phi})\) be a reparameterizable belief distribution with samples
\(\boldsymbol{\theta}_i=g(\boldsymbol{\epsilon}_i,\boldsymbol{\phi})\),
\(\boldsymbol{\epsilon}_i\sim p(\boldsymbol{\epsilon})\), and let
\(C:\mathbb{R}^d\to\mathbb{R}\) be continuously differentiable with
\(\mathbb{E}[|\nabla_{\boldsymbol{\phi}}C(g(\boldsymbol{\epsilon},\boldsymbol{\phi}))|]<\infty\).
Then the gradient of the soft CVaR estimator in~\eqref{eq:cvar_soft} satisfies
\begin{equation}\label{eq:cvar_phi_grad}
\begin{array}{l}
\nabla_{\boldsymbol{\phi}}\,\widehat{\mathrm{CVaR}}_\beta
= \dfrac{1}{(1-\beta)N}
\displaystyle\sum_{i=1}^{N}
\sigma_{\kappa_\rho}\!\left(C(\boldsymbol{\theta}_i)-\hat{\eta}\right) \\[4pt]
\qquad{}\times\nabla_{\boldsymbol{\phi}}C(g(\boldsymbol{\epsilon}_i,\boldsymbol{\phi})).
\end{array}
\end{equation}
\end{theorem}

\begin{proof}
By the reparameterization trick, each sample \(\boldsymbol{\theta}_i=g(\boldsymbol{\epsilon}_i,\boldsymbol{\phi})\)
is a deterministic differentiable function of \(\boldsymbol{\phi}\) for fixed \(\boldsymbol{\epsilon}_i\). The softplus \(\mathsf{SoftPlus}(x;\kappa_\rho)=\kappa_\rho^{-1}\log(1+\exp(\kappa_\rho x))\) is smooth with derivative \(\sigma_{\kappa_\rho}(x)=(1+\exp(-\kappa_\rho x))^{-1}\). Since \(C\) and \(g\) are differentiable by assumption, applying the chain rule gives
\begin{equation}\label{eq:cvar_phi_chain}
\begin{aligned}
&
\nabla_{\boldsymbol{\phi}}\!\left[
\hat{\eta}+
\frac{1}{(1-\beta)N}\sum_{i=1}^{N}
\mathsf{SoftPlus}\!\left(C(g(\boldsymbol{\epsilon}_i,\boldsymbol{\phi}))-\hat{\eta};\kappa_\rho\right)
\right] \\
&\equiv
\frac{1}{(1-\beta)N}\sum_{i=1}^{N}
\sigma_{\kappa_\rho}\!\left(C(\boldsymbol{\theta}_i)-\hat{\eta}\right)
\nabla_{\boldsymbol{\theta}}C(\boldsymbol{\theta}_i)
\nabla_{\boldsymbol{\phi}}g(\boldsymbol{\epsilon}_i,\boldsymbol{\phi}).
\end{aligned}
\end{equation}
The gradient contribution through \(\hat{\eta}\) vanishes at the optimal quantile by the envelope theorem~\cite{dingEnvelopeQuantileRegression2019,lafranceEnvelopeTheoremDynamic1991}, which yields~\eqref{eq:cvar_phi_grad}.
\end{proof}

\begin{remark}
The bias of the softplus approximation relative to the exact hinge is bounded
by \(\kappa_\rho^{-1}\log 2\), which vanishes as \(\kappa_\rho\to\infty\). We
treat \(\hat{\eta}\) as fixed with respect to \(\boldsymbol{\phi}\) in the
stochastic estimator; alternatively, optimizing jointly over \(\eta\) recovers
the envelope-theorem gradient~\cite{rockafellar2000optimization}.
\end{remark}

\subsection{Implicit Belief Variant and Hyperparameters}\label[appendix]{apx:belief_nets}
When the posterior is highly multimodal or shaped by complex contact geometry,
a Gaussian mixture may require an impractically large \(K\). As an alternative
high-capacity representation, we parameterize the unnormalized log-density
directly with a sinusoidal representation network (SIREN)~\cite{sitzmann2020implicit}:
\begin{equation}\label{eq:sirennet}
p(\boldsymbol{\theta}\mid\boldsymbol{\omega})
=
\frac{\exp\!\bigl(f_{\boldsymbol{\omega}}(\boldsymbol{\theta})\bigr)}
{Z(\boldsymbol{\omega})},
\end{equation}
where \(f_{\boldsymbol{\omega}}\) uses sinusoidal activations
\(\sigma(x)=\sin(\omega_0 x)\), and \(Z(\boldsymbol{\omega})\) is the partition function. Sampling is performed via unadjusted Langevin dynamics with step size \(\alpha\):
\begin{equation}\label{eq:langevin_update}
\boldsymbol{\theta}_{j+1}
= \boldsymbol{\theta}_{j}
+ \frac{\alpha}{2}\nabla_{\boldsymbol{\theta}}\!\log p(\boldsymbol{\theta}_{j}\mid\boldsymbol{\omega}) +
\sqrt{\alpha}\,\boldsymbol{\epsilon}_{j}, \quad
\boldsymbol{\epsilon}_{j}\sim\mathcal{N}(\mathbf{0},\mathbf{I}).
\end{equation}
Because the score \(\nabla_{\boldsymbol{\theta}}\log p\) is computed by automatic differentiation through \(f_{\boldsymbol{\omega}}\), the resulting samples are approximately differentiable with respect to \(\boldsymbol{\omega}\). Under standard smoothness assumptions on \(f_{\boldsymbol{\omega}}\), with Lipschitz gradient constant \(L\), the total-variation distance (\(\mathrm{TV}(\cdot, \star)\)) between the Langevin chain distribution \(p_J\) and the target distribution \(p^\star\)
satisfies~\cite{xuGlobalConvergenceLangevin2018}
\begin{equation}\label{eq:langevin_tv}
\mathrm{TV}(p_J,p^\star)
\le \sqrt{\frac{d\alpha L}{4}} +
\exp(-\alpha\lambda J/2)\,\mathrm{TV}(p_0,p^\star),
\end{equation}
where \(\lambda\) is the log-Sobolev constant of the target. In our implementation we use \(\alpha=10^{-3}\), \(J=50\), and warm-start the chain from the previous belief's samples. The approximate reparameterization guarantee follows from treating the Langevin chain as a differentiable function of \(\boldsymbol{\omega}\), with bias controlled by the discretization error in~\eqref{eq:langevin_tv}.

\begin{table}[!t]
\centering
\setlength{\abovecaptionskip}{2pt}
\setlength{\belowcaptionskip}{-3pt}
\caption{Architecture details and hyperparameters.}
\label{tab:hyperparams_full}
\scriptsize
\setlength{\tabcolsep}{2.5pt}
\renewcommand{\arraystretch}{0.78}
\begin{tabular}{@{}p{0.31\columnwidth}L{0.65\columnwidth}@{}}
\toprule
\textbf{Group} & \textbf{Details} \\
\midrule
Belief Nets &
\(d_h=64\), \(d_a=6\), \(d_o=41\), \(d=26\);
transition MLP \([d_h+d_a,128,128,d_h]\);
observation MLP \([d_h+d_o,128,128,d_h]\);
decoder MLP \([d_h,64,K(2d+1)]\);
\(K=8\) for \vnb, \(K=1\) for \gauss/\cem;
SIREN \(3\times128\), \(\omega_0=30\). \\
\midrule
Observation Layout &
(41) pose 6, \(\mathrm{tr}(\boldsymbol{\Sigma}^{\mathrm{pose}})\) 1,
tactile forces 5, contacts 5, joint positions 11,
joint velocities 11, visual scores 2. \\
\midrule
Latent Layout &
(26) pose 6 and \(5\times\) contact parameters \(\times4\). \\
\midrule
Training &
Adam, \(\eta=3\times10^{-4}\), batch size 64;
Gumbel-Softmax \(\tau:1.0\to0.1\) cosine over 500 epochs;
\(\kappa_\rho=5.0\), \(\kappa_f=100.0\);
\(\beta\in\{0.5,0.9,0.95,0.99\}\);
belief samples \(N=256\) online, 512 final evaluation. \\
\midrule
MPC/Planning &
\(T_{\max}=80\), PF particles 100;
CEM population/elite/iterations \(64/0.2/3\);
20 action candidates, close \(\in\{0.05,0.10,0.15,0.20\}\);
\(\lambda_c=1.0\) \vnb, \(0.5\) \gausscvar, \(0\) \gauss;
\(\lambda_v=0.3\). \\
\midrule
Cost Weights &
\(\alpha_s=1.0\), \(\alpha_g=0.5\),
\(\alpha_r=2.0\), \(\alpha_n=0.1\). \\
\midrule
Visual Cost &
\(\omega_{\mathrm{pose}}=\omega_{\mathrm{occ}}=\omega_{\mathrm{seg}}=1.0\);
\(\sigma_{\mathrm{base}}=0.005\,\mathrm m\);
occlusion radius \(0.05\,\mathrm m\). \\
\midrule
Perturbations &
lateral forces \(\{3,5,8,12\}\,\mathrm N\), 4 directions;
torques \(\{0.3,0.6,1.0\}\,\mathrm{Nm}\), 3 axes;
\(\mu\to\{0.05,0.10,0.15\}\);
drop/displacement thresholds \(0.015/0.04\,\mathrm m\);
rotation threshold \(0.3\,\mathrm{rad}\). \\
\midrule
Compute &
NVIDIA RTX 4070, 12\,GB;
\vnb\ episode time \(\approx4\)--6\,s;
\pf\ episode time \(\approx65\)--72\,s. \\
\bottomrule
\end{tabular}
\vspace{-4pt}
\end{table}

\subsection{ICP Scoring and Pose Perturbation Details}\label[appendix]{apx:experiments}
For hardware pose estimation, we compute bounding-box extents for the observed
point-cloud cluster \(\mathcal{P}_{\mathrm{obs}}\) returned by the Orbbec Astra Pro Plus and the object's CAD model. We generate 14 candidate rotations from
cardinal-axis permutations and \(\pi\)-radian rotations, run ICP from each candidate, and score each registration using
\begin{equation}\label{eq:icp_score}
s=s_{\mathrm{fit}}-\lambda_s e_{\mathrm{rmse}},
\vspace{-4pt}
\end{equation}
where \(s_{\mathrm{fit}}\) is the ICP fitness score, \(e_{\mathrm{rmse}}\) is
the position-only root-mean-square error in meters, and \(\lambda_s=0.1\). We select the registration with the highest score and compute the final object
rotation as \(\mathbf{R}_{o}=\mathbf{R}_{\mathrm{icp}}^\top\mathbf{R}_{\mathrm{init}}\),
where \(\mathbf{R}_{\mathrm{init}}\) is the CAD model's default rotation. The final pose estimate stores the world-frame \((\mathcal{F}_{\mathcal{R}})\) position and orientation \(\mathbf{R}_{o}\), and its covariance enters the MPC objective \eqref{eq:mpcobj} through \(C_{\mathrm{vis}}\).

\begin{figure}[t]
\centering
\setlength{\fboxsep}{0pt}
\setlength{\fboxrule}{1pt}
\setlength{\tabcolsep}{1pt}

\begin{minipage}[b]{0.485\textwidth}
\centering
\begin{minipage}[b]{0.52\linewidth}
    \centering
    \adjustbox{angle=-90, width=\linewidth}{%
        \includestandalone{figures/tikz/jig_figure}}
    \par\vspace{1pt}{\small (a) Schematic}
\end{minipage}
\hfill
\begin{minipage}[b]{0.42\linewidth}
    \centering
    \fbox{\includegraphics[
        width=\dimexpr\linewidth-2\fboxrule\relax,
        trim=0pt 0pt 0pt 0pt,
        clip
    ]{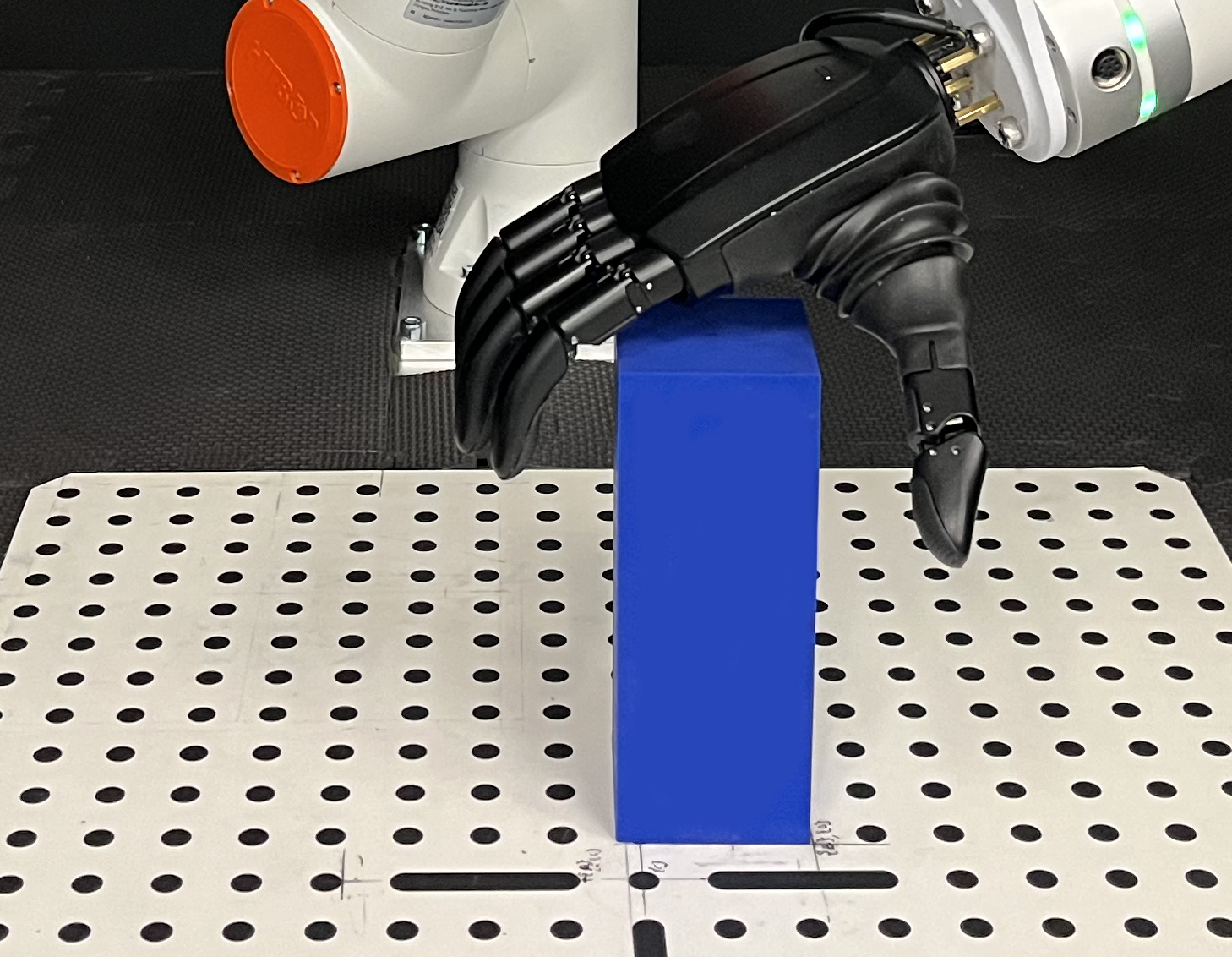}}
    \par\vspace{1pt}{\small (b) Hardware Setup}
\end{minipage}
\vspace{-2pt}
\captionof{figure}{\textbf{Pose Perturbation Protocol.}
A rigid 24\(\times\)19-inch pegboard defines a discrete planar coordinate frame
\(\{\mathcal{F}_{\mathcal{B}}\}\), with origin \(\mathbf{p}_{\mathcal{B}}\)
axis-aligned to the robot base frame \(\{\mathcal{F}_{\mathcal{R}}\}\).
We place objects at fixed dot offsets from \(\mathbf{p}_{\mathcal{B}}\), and apply
controlled pose perturbations using the offset set in~\eqref{eq:hw_offs}.}
\label{fig:jig}
\end{minipage}
\\[1em]
\begin{minipage}[b]{0.465\textwidth}
\centering
\includegraphics[width=\linewidth]{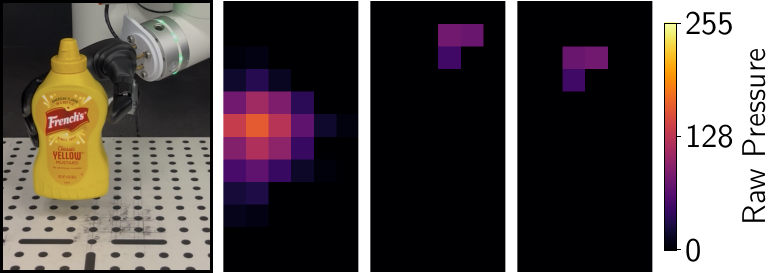}
\begin{tabular}{@{}ccccc@{}}
\par\vspace{1pt}
\makebox[0.27\linewidth]{\small (a)} &
\makebox[0.17\linewidth]{\small (b)} &
\makebox[0.17\linewidth]{\small (c)} &
\makebox[0.17\linewidth]{\small (d)} &
\makebox[0.14\linewidth]{}
\end{tabular}
\vspace{-2pt}
\captionof{figure}{\textbf{Hardware Grasp and Tactile Readings.} Sample hardware grasp~(a) and corresponding \(12\times 6\) taxel arrays of analog pressure values in \([0, 255]\) for the active thumb~(b), index~(c), and middle finger~(d).}
\label{fig:tactdesc}
\end{minipage}
\vspace{-6pt}
\end{figure}
The pegboard is oriented with its longer edge, 24 dots, aligned with the
\(y_{\mathcal{R}}\) axis and its shorter edge, 19 dots, aligned with the
\(x_{\mathcal{R}}\) axis. The pegboard center \(\mathbf{p}_{\mathcal{B}}\) is
axis-aligned with the robot base frame. We assign each object an initial lattice
location relative, in lattice counts, to the lattice at \(\mathbf{p}_{\mathcal{B}}\):
box and soup can \((-2,+1)\), and mustard bottle \((-1,+1)\). The four
evaluated offsets are the set \(\mathcal{O}\) in~\eqref{eq:hw_offs}.
For each trial, we compute the realized \(SE(3)\) perturbation \(\Delta X\)
from the ICP estimate, then use the RealSense stream to monitor slip during the
4\,cm lift and 2\,s hold. A trial is successful if object displacement remains
below 2\,cm.

\subsection{Hardware Grasp Quality Proxy}\label[appendix]{apx:ehatcomp}
The robotic hand used in our hardware experiments provides a \(12 \times 6\)
pressure array of taxels, tactile pixels, at each fingertip; see
\Cref{fig:tactdesc}. Each taxel reports pressure in \([0,255]\), and we sum over
the tactile-pad area to obtain a normal-force proxy
\begin{equation}\label{eq:tactforce}
    \hat{f}_{i} = \frac{1}{k_{i}} \sum_{u=1}^{6}\sum_{v=1}^{12} p_{i,u,v},
\end{equation}
where \(p_{i,u,v} \in [0,255]\) is taxel \((u,v)\) at finger \(i\), and
\(k_{i}\) is a finger-specific pressure-to-force calibration constant
\((k_{\text{thumb}}=25.5,\; k_{j}=204.0\) for \(j \neq \text{thumb})\). A
finger is in contact when \(\hat{f}_{i}>f_{\min}=0.05\;\text{N}\). However, because the tactile array lacks the full 6D contact wrench, our hardware
pipeline cannot directly certify force closure. We therefore use a
force-closure-inspired tactile grasp-quality proxy computed from active contact
geometry. Let \(\widehat{\mathcal{C}}_t\) denote the hardware-estimated active
contact set at time \(t\), obtained from tactile activation. This is the
hardware analogue of the contact set used to construct the grasp wrench space
in the model-based planner. For each active finger
\(i \in \widehat{\mathcal{C}}_t\), we estimate the fingertip world position
\(\mathbf{p}_{i,t}\) via forward kinematics and construct the unit contact normal
\begin{equation}\label{eq:contact-normal}
    \hat{\mathbf{n}}_{i,t}
    =
    \frac{\mathbf{c}_t - \mathbf{p}_{i,t}}
         {\|\mathbf{c}_t - \mathbf{p}_{i,t}\|},
\end{equation}
where \(\mathbf{c}_t \in \mathbb{R}^{3}\) is the FK-inferred object center. We
then complete the right-handed contact frame
\(\mathcal{F}^{\mathcal{C}}_{i,t}
            =\{\hat{\mathbf n}_{i,t},
            \hat{\mathbf t}_{i,t}^{(1)},
            \hat{\mathbf t}_{i,t}^{(2)}\}\)
by choosing
\(\hat{\mathbf{t}}_{i,t}^{(1)} = \hat{\mathbf{n}}_{i,t} \times \mathbf{e}\), with
\(\mathbf{e}\) a non-parallel reference direction, and
\(\hat{\mathbf{t}}_{i,t}^{(2)} = \hat{\mathbf{n}}_{i,t} \times \hat{\mathbf{t}}_{i,t}^{(1)}\).
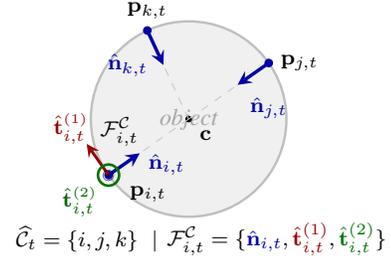
\begin{figure}[!t]
        \centering
        \setlength{\abovecaptionskip}{1pt}
        \setlength{\belowcaptionskip}{-4pt}
        \begin{tikzpicture}[
            scale=0.8,
            >=stealth,
            nvec/.style={->, very thick, blue!65!black, line width=1.3pt},
            tvec/.style={->, thick,      red!60!black,  line width=1.1pt},
            oop/.style={circle, draw=green!45!black, fill=none,
                        line width=1pt, minimum size=8pt, inner sep=0pt},
            lbl/.style={font=\footnotesize},
        ]
        \def\R{2.0}
        \def\nlen{0.78}
        \def\tlen{0.78}

        \draw[fill=gray!12, draw=gray!50, line width=0.9pt] (0,0) circle (\R);
        \node[lbl, gray!70, font=\small\itshape] at (0,0) {object};

        \filldraw[black] (0,0) circle (1.6pt);
        \node[lbl, below right=1pt] at (0,0) {\(\mathbf{c}_t\)};

        \pgfmathsetmacro{\ATH}{215}
        \pgfmathsetmacro{\QTHx}{\R*cos(\ATH)}
        \pgfmathsetmacro{\QTHy}{\R*sin(\ATH)}
        \pgfmathsetmacro{\NTHx}{-cos(\ATH)}
        \pgfmathsetmacro{\NTHy}{-sin(\ATH)}
        \pgfmathsetmacro{\TTHx}{-\NTHy}
        \pgfmathsetmacro{\TTHy}{\NTHx}
        \pgfmathsetmacro{\NTHex}{\QTHx + \nlen*\NTHx}
        \pgfmathsetmacro{\NTHey}{\QTHy + \nlen*\NTHy}
        \pgfmathsetmacro{\TTHex}{\QTHx + \tlen*\TTHx}
        \pgfmathsetmacro{\TTHey}{\QTHy + \tlen*\TTHy}
        \pgfmathsetmacro{\OOPx}{\QTHx + 0.42*\TTHx + 0.30*\NTHx}
        \pgfmathsetmacro{\OOPy}{\QTHy + 0.42*\TTHy + 0.30*\NTHy}

        \draw[dashed, gray!45, very thin] (0,0) -- (\QTHx,\QTHy);
        \filldraw[blue!65!black] (\QTHx,\QTHy) circle (2.2pt);
        \node[lbl, below right =0.8pt, xshift=4pt] at (\QTHx,\QTHy) {\(\mathbf{p}_{i,t}\)};
        \node[lbl, above=-2pt, xshift=4pt] at (\OOPx,\OOPy) {\(\mathcal{F}^{\mathcal{C}}_{i,t}\)};
        \draw[nvec] (\QTHx,\QTHy) -- (\NTHex,\NTHey)
        node[lbl, below right=-1pt, yshift=3pt, at end] {\(\hat{\mathbf{n}}_{i,t}\)};

        \draw[tvec] (\QTHx,\QTHy) -- (\TTHex,\TTHey)
        node[lbl, above=0pt, xshift=-6pt, yshift=-4pt, at end]
        {\(\hat{\mathbf{t}}_{i,t}^{\scriptscriptstyle(1)}\)};

        \node[oop, font=\tiny] at (\QTHx,\QTHy) {\(\odot\)};

        \node[lbl, green!40!black, below left=1pt, fill=none] at (\QTHx,\QTHy)
            {\(\hat{\mathbf{t}}_{i,t}^{\scriptscriptstyle(2)}\)};

        \pgfmathsetmacro{\AIN}{35}
        \pgfmathsetmacro{\QINx}{\R*cos(\AIN)}
        \pgfmathsetmacro{\QINy}{\R*sin(\AIN)}
        \pgfmathsetmacro{\NINx}{-cos(\AIN)}
        \pgfmathsetmacro{\NINy}{-sin(\AIN)}
        \pgfmathsetmacro{\NINex}{\QINx + \nlen*\NINx}
        \pgfmathsetmacro{\NINey}{\QINy + \nlen*\NINy}

        \draw[dashed, gray!45, very thin] (0,0) -- (\QINx,\QINy);
        \filldraw[blue!65!black] (\QINx,\QINy) circle (2.2pt);
        \node[lbl, right=1pt] at (\QINx,\QINy) {\(\mathbf{p}_{j,t}\)};
        \draw[nvec] (\QINx,\QINy) -- (\NINex,\NINey)
            node[lbl, below right=1pt, at end] {\(\hat{\mathbf{n}}_{j,t}\)};

        \pgfmathsetmacro{\AMD}{115}
        \pgfmathsetmacro{\QMDx}{\R*cos(\AMD)}
        \pgfmathsetmacro{\QMDy}{\R*sin(\AMD)}
        \pgfmathsetmacro{\NMDx}{-cos(\AMD)}
        \pgfmathsetmacro{\NMDy}{-sin(\AMD)}
        \pgfmathsetmacro{\NMDex}{\QMDx + \nlen*\NMDx}
        \pgfmathsetmacro{\NMDey}{\QMDy + \nlen*\NMDy}

        \draw[dashed, gray!45, very thin] (0,0) -- (\QMDx,\QMDy);
        \filldraw[blue!65!black] (\QMDx,\QMDy) circle (2.2pt);
        \node[lbl, above=1pt] at (\QMDx,\QMDy) {\(\mathbf{p}_{k,t}\)};
        \draw[nvec] (\QMDx,\QMDy) -- (\NMDex,\NMDey)
            node[lbl, left=2pt, at end] {\(\hat{\mathbf{n}}_{k,t}\)};

        \node[lbl] at (-2,-2.45)
            {\(\widehat{\mathcal{C}}_t=\{i,j,k\} ~ \mid ~\)};

       \node[lbl, align=center] at (1.8,-2.45)
        {\(\mathcal{F}^{\mathcal{C}}_{i,t}
        =\{\textcolor{blue!65!black}{\hat{\mathbf n}_{i,t}},
        \textcolor{red!60!black}{\hat{\mathbf t}_{i,t}^{(1)}},
        \textcolor{green!45!black}{\hat{\mathbf t}_{i,t}^{(2)}}\}\)};
        \end{tikzpicture}
    \vspace{-3pt}
    \caption{\textbf{Contact-Frame Assignment.} For each active finger \(i\in\widehat{\mathcal C}_t\), we compute each fingertip's position, \(\mathbf{p}_{i,t}\), via forward kinematics. Dashed lines connect the inferred object center \(\mathbf{c}_t\) to active fingertip positions. Blue arrows denote inward contact normals \(\hat{\mathbf n}_{i,t}\), the red arrow denotes \(\smash{\hat{\mathbf t}_{i,t}^{(1)}}\), and \(\odot\) denotes \(\smash{\hat{\mathbf t}_{i,t}^{(2)}}\), pointing out of the page.}
    \label{fig:ctcframe}
    \vspace{-1pt}
\end{figure}
With \(\mathcal{F}^{\mathcal{C}}_{i,t}\), we then compute a force-closure-inspired tactile proxy as
\begin{equation}\label{eq:ehat}
    \hat{\varepsilon}_t
    \;=\;
    \mathcal{Q}\!\Bigl(
    \bigl\{(\mathbf{p}_{i,t},\,\hat{\mathbf{n}}_{i,t},\,\mu)\bigr\}_{i \in \widehat{\mathcal{C}}_t}
    \Bigr),
\end{equation}
where \(\mathcal{Q}(\cdot)\) is the Ferrari--Canny quality
functional~\cite{ferrari1992planning} evaluated on the approximate grasp wrench
space induced by inferred contacts, linearized friction cones, fixed
\(\mu=0.5\), and unit normal forces. Since the contacts, normals, and friction
are estimated, \(\hat{\varepsilon}_t\) is a force-closure-inspired tactile proxy
rather than a direct certificate. We set \(\hat{\varepsilon}_t=0\) when
\(|\widehat{\mathcal{C}}_t|<2\), since the approximate grasp wrench space is
then not spanned.

\end{document}